\documentclass{article}

\usepackage{microtype}
\usepackage{graphicx}
\usepackage{subfigure}
\usepackage{booktabs} 
\usepackage{amsmath}
\usepackage{amsfonts}
\usepackage{amssymb}
\usepackage{amsthm}
\usepackage{enumitem}   
\usepackage{physics}
\usepackage{mathrsfs}
\usepackage{xcolor}
\usepackage{multicol}
\usepackage{accents}

\usepackage{hyperref}
\usepackage{cleveref}

\usepackage{nicefrac}

\usepackage[accepted]{icml2023}

\usepackage{general-private}
\usepackage{special-private}
\usepackage{thm-config}

\icmltitlerunning{Tight and fast generalization error bound of graph embedding in metric space.}

\begin{document}

\twocolumn[
\icmltitle{Tight and fast generalization error bound of graph embedding in metric space.}



\icmlsetsymbol{equal}{*}

\begin{icmlauthorlist}
\icmlauthor{Atsushi Suzuki}{kcl,utokyo}
\icmlauthor{Atsushi Nitanda}{kyutech,riken}
\icmlauthor{Taiji Suzuki}{utokyo,riken}
\icmlauthor{Jing Wang}{gre}
\icmlauthor{Feng Tian}{dku}
\icmlauthor{Kenji Yamanishi}{utokyo}
\end{icmlauthorlist}

\icmlaffiliation{utokyo}{The Univerity of Tokyo, Japan}
\icmlaffiliation{riken}{RIKEN, Japan}
\icmlaffiliation{kcl}{King's College London, UK}
\icmlaffiliation{kyutech}{Kyushu Institute of Technology, Japan}
\icmlaffiliation{gre}{University of Greenwich, UK}
\icmlaffiliation{dku}{Duke Kunshan University, China}
\icmlcorrespondingauthor{Atsushi Suzuki}{atsushi.suzuki.rd@gmail.com}

\icmlkeywords{Machine Learning, ICML}

\vskip 0.3in
]



\printAffiliationsAndNotice{}  

\begin{abstract}
Recent studies have experimentally shown that we can achieve in non-Euclidean metric space effective and efficient graph embedding, which aims to obtain the vertices' representations reflecting the graph's structure in the metric space.
Specifically, graph embedding in hyperbolic space has experimentally succeeded in embedding graphs with hierarchical-tree structure, e.g., data in natural languages, social networks, and knowledge bases.
However, recent theoretical analyses have shown a much higher upper bound on non-Euclidean graph embedding's generalization error than Euclidean one's, where a high generalization error indicates that the incompleteness and noise in the data can significantly damage learning performance. 
It implies that the existing bound cannot guarantee the success of graph embedding in non-Euclidean metric space in a practical training data size, which can prevent non-Euclidean graph embedding's application in real problems.
This paper provides a novel upper bound of graph embedding's generalization error by evaluating the local Rademacher complexity of the model as a function set of the distances of representation couples.
Our bound clarifies that the performance of graph embedding in non-Euclidean metric space, including hyperbolic space, is better than the existing upper bounds suggest.
Specifically, our new upper bound is polynomial in the metric space's geometric radius $R$ and can be $O(\frac{1}{S})$ at the fastest, where $S$ is the training data size.
Our bound is significantly tighter and faster than the existing one, which can be exponential to $R$ and $O(\frac{1}{\sqrt{S}})$ at the fastest.
Specific calculations on example cases show that graph embedding in non-Euclidean metric space can outperform that in Euclidean space with much smaller training data than the existing bound has suggested.
\end{abstract}

\section{Introduction}
\label{sec:Intro}
Graphs are a fundamental form of real-world entities and their relations, such as words in natural languages, people in social networks, and objects in knowledge bases.
Here, the vertices and edges of a graph correspond to the entities and the relations among them, respectively.
Based on the formulation, \NewTerm{graph embedding}, learning representations of the graph's vertices in a metric space has enabled numerous applications for those data, such as machine translation and sentiment analysis for natural language \citep{DBLP:conf/nips/MikolovSCCD13, DBLP:conf/emnlp/PenningtonSM14, DBLP:journals/tacl/BojanowskiGJM17, DBLP:conf/iclr/TifreaBG19}, and community detection and link prediction for social network data \citep{hoff2002latent, DBLP:conf/kdd/PerozziAS14, DBLP:conf/www/TangQWZYM15, DBLP:conf/kdd/TangQM15, DBLP:conf/kdd/GroverL16}, pathway prediction of biochemical network \citep{dale2010machine, ma2021leveraging}, and link prediction and triplet classification for knowledge base \citep{DBLP:conf/icml/NickelTK11, DBLP:conf/nips/BordesUGWY13, DBLP:conf/naacl/RiedelYMM13, DBLP:conf/aaai/NickelRP16, DBLP:conf/icml/TrouillonWRGB16, DBLP:conf/aaai/EbisuI18}.
The metric space where we get representations of the vertices is called the \NewTerm{representation space} in this paper.
Graph embedding aims to obtain representations such that the metric reflects the relations defined by the edges.
Specifically, we expect the representations of a couple of vertices to be close if they are connected and distant if not.

It is essential in the representation learning context to discuss a generic metric space, not only Euclidean space, as a representation space, although Euclidean space or the inner product space has been widely used \citep{DBLP:conf/nips/MikolovSCCD13, DBLP:conf/emnlp/PenningtonSM14, DBLP:journals/tacl/BojanowskiGJM17, hoff2002latent, DBLP:conf/kdd/PerozziAS14, DBLP:conf/www/TangQWZYM15, DBLP:conf/kdd/TangQM15, DBLP:conf/kdd/GroverL16}.
It is because many studies have experimentally shown the effectivity or representation learning in non-Euclidean metric space, in particular, hyperbolic space \citep{DBLP:conf/nips/NickelK17, DBLP:conf/icml/GaneaBH18, DBLP:conf/icml/SalaSGR18, DBLP:conf/nips/GaneaBH18, DBLP:conf/nips/ChamiYRL19, DBLP:conf/iclr/GulcehreDMRPHBB19, DBLP:conf/iclr/TifreaBG19, DBLP:conf/nips/BalazevicAH19} since hyperbolic space can represent a graph with a hierarchical tree-like structure with arbitrarily small approximation error \citep{gromov1987hyperbolic, DBLP:conf/gd/Sarkar11, DBLP:conf/icml/SalaSGR18}.
This advantage comes from the property that the volume of hyperbolic space grows exponentially in its radius $\Radius$.
This is in contrast to Euclidean space, which has limitations in representing such a graph \citep{DBLP:conf/uist/LampingR94, ritter1999self, DBLP:conf/nips/NickelK17}.

The above facts motivate us to use graph embedding in non-Euclidean space actively.
However, the ability of some non-Euclidean space to represent a complex graph may lead to overfitting in graph embedding settings, where data are incomplete or noisy.
It is because there is a trade-off between a model's representability and the potential of overfitting in general machine learning settings.
Hence, in order for graph embedding users to select the best model, we need to evaluate each model's \NewTerm{generalization error}, that is, how much the model's performance is badly influenced by incompleteness and noise of the data.
Indeed, recent research \citep{suzuki2021generalization,suzuki2021graph}, for the first time, has provided upper bounds of representation learning in non-Euclidean space by converting the graph embedding problem to a linear discrimination analysis problem from Gramian matrices in the inner-product space or Minkowski space. 
Their results suggest that the generalization error of the representation learning's performance could be exponential in the radius $\Radius$ of the space we use in hyperbolic space.
This bound is in line with the volume of the space.
Their evaluation implies that we might need an impractically large data size (\EG $> 10^{72}$ as in \Rem \ref{rem:comparison}) to get a better performance graph embedding in hyperbolic space than in Euclidean space.
Nevertheless, the following observations imply the existing bounds overestimate the generalization error.
\begin{itemize}[leftmargin=*,topsep=0pt]
\item They do not consider the metric space's property.
Even if the volume of a ball grows exponentially in its radius $\Radius$ as in hyperbolic space, the distance between two points in the ball is always smaller than $2 \Radius$. 
Hence, the generalization error might avoid an exponential dependency on the space's radius.
\item They do not use the ``local'' model complexity around the optimal representations, resulting in a convergence rate $O \qty(\nicefrac{1}{\sqrt{\NData}})$ for data size $\NData$.
According to past research \citep{bartlett2005local, koltchinskii2006local} in the learning theory context, the generalization error can be $O \qty(\nicefrac{1}{\NData})$ if the complexity of the ``neighborhood'' of the best hypothesis function is limited.
In the graph embedding setting, the model is substantially finite-dimensional since there are finite representation couples only.
Hence, it is highly possible that the ``local'' complexity is small enough.
\end{itemize}
Based on the above observation, we aim to derive a tighter and faster generalization error bound of graph embedding in metric space.
The above observations imply that we have the potential to achieve a tighter and faster bound if we regard graph embedding's loss function as a function of the distance values of the finite representation couples.
Indeed, we have achieved the aim by reformulating graph embedding's loss function as a restriction of the composition of a non-linear function and a linear function of the distances of pairs of representations.
Specifically, our contributions are the following:
\begin{itemize}[leftmargin=*,topsep=0pt]
\item We have derived a novel upper bound of the \NewTerm{Rademacher complexity} \citep{DBLP:journals/tit/Koltchinskii01, koltchinskii2000rademacher, DBLP:journals/ml/BartlettBL02} of graph embedding's hypothesis function set and its local subset, called the \NewTerm{local Rademacher complexity} \citep{bartlett2005local, koltchinskii2006local}. 
The bound is tighter than existing ones for most cases since it is polynomial for the representation space's radius if the space is metric. 
The Rademacher complexity evaluation can apply to representation learning settings discussed in the past papers \cite{DBLP:conf/nips/JainJN16, gao2018importance, suzuki2021generalization, suzuki2021graph} since their bounds were also derived from the Rademacher complexity evaluation. 
\item Based on the above global and local Rademacher complexity bound, we have derived a novel upper bound of graph embedding's generalization error. 
Our bound is tighter in that it is polynomial for the representation space's radius $\Radius$ if the space is metric and faster in that it is $O \qty(\nicefrac{1}{\NData})$ at the fastest than the existing $O \qty(\nicefrac{1}{\sqrt{\NData}})$ bound.
\item We have calculated specific bounds for graph embedding in Euclidean and hyperbolic spaces and derived a significantly improved upper bound of the data size that the graph embedding in hyperbolic space needs to outperform that in Euclidean space when the graph is a tree. 
\end{itemize}

The remainder of the paper is organized as follows. 
\Sec \ref{sec:Preliminary} formulates the graph embedding in the learning theory style. 
\Sec \ref{sec:Main} gives our assumptions and generalization error bounds, the main result of the paper.
\Sec \ref{sec:example} provides examples of the application of our main result.
\Sec \ref{sec:proof} gives the core technical result to enable comparisons to previous work and discussions on potential future work.
\Sec \ref{sec:comparison} compares our result with previous work based on \Sec \ref{sec:proof}.
\Sec \ref{sec:FutureWork} discusses potential future work.

\section{Preliminaries}
\label{sec:Preliminary}
\paragraph{Notation}
The symbol $\DefEq$ indicates that its left side is defined by its right side. 
We denote by $\Integer, \Integer_{>0}, \Real, \Real_{\ge 0}$ the set of integers, the set of positive integers, the set of real numbers, and the set of non-negative real numbers, respectively. 
For $\NAxes, \in \Integer_{>0}$, $\Real^{\NAxes}$ denotes the set of $\NAxes$-dimensional real vectors. 
For $\PointVec \in \Real^{\NAxes}$, $\PointVec^{\Transpose}$ indicates its transpose.
$\Sign: \Real \to \qty{0, \pm 1}$ is the sign function defined by $\Sign \qty(\Rate) = -1$ if $\Rate < 0$, $\Sign \qty(\Rate) = +1$ if $\Rate > 0$, and $\Sign \qty(\Rate) = 0$ if $\Rate = 0$.
For $\Rate, \Rate' \in \Real$, we define $\Rate \land \Rate' \DefEq \min \qty{\Rate, \Rate'}$ and $\Rate \lor \Rate' \DefEq \max \qty{\Rate, \Rate'}$.
For a finite set $\EntitySet$, we denote the number of elements in $\EntitySet$ by $\abs{\EntitySet} \in \Integer_{\ge 0}$, and we denote the set of two element subsets of $\EntitySet$ by $\PairSet{\EntitySet}$, 
\IE $\PairSet{\EntitySet} \DefEq \qty{\SetI \subset \EntitySet \middle| \abs{\SetI} = 2}$.
Note that $\abs{\PairSet{\EntitySet}} = \frac{\abs{\EntitySet} \qty(\abs{\EntitySet} - 1)}{2}$ holds.
For sets $\SetI$ and $\SetII$, 
we denote by $2^{\SetI}$ the power set on $\SetI$, and by $\SetII^{\SetI}$ the set of maps from $\SetI$ to $\SetII$.
For example, $\Real^{\SetI}$ denotes the set of real functions on $\SetI$.
If $\SetI$ is a measurable space, we denote the set of measurable functions on $\SetI$ by $\mathcal{L}_{0} \qty(\SetI)$.
We denote the expectation with respect to a random variable $\Point$ that follows a distribution $\Measure$ by $\Expect_{\Point \sim \Measure}$.

\subsection{True dissimilarity}
\label{sub:OE}
First, we formulate the representation learning from pair-label couples, which is of interest in this paper.
This includes graph embedding as a special case.
Let $\EntitySet$ denote the entity set. 
We assume that there exists a \NewTerm{true dissimilarity} function $\Dsim^{*}: \EntitySet \times \EntitySet \to \Real_{\ge 0}$, where $\Dsim^{*} \qty(\EntityI, \EntityII)$ indicates the true dissimilarity between entity $\EntityI$ and entity $\EntityII$.
The entities $\EntityI$ and $\EntityII$ are ``similar'' or strongly related if $\Dsim^{*} \qty(\EntityI, \EntityII)$ is small and ``dissimilar'' or weakly related if $\Dsim^{*} \qty(\EntityI, \EntityII)$ is large.
Specifically, we fix a threshold $\Threshold \in \Real$ and say $\EntityI$ and $\EntityII$ are similar if $\Dsim^{*} \qty(\EntityI, \EntityII) < \Threshold$ and dissimilar if $\Dsim^{*} \qty(\EntityI, \EntityII) > \Threshold$.
Note that $\Dsim^{*} \qty(\EntityI, \EntityII) = \Threshold$ holds with probability at most zero in this paper, so we can ignore this corner case.
Throughout this paper, we assume the symmetry of the dissimilarity function, \IE $\Dsim^{*} \qty(\EntityI, \EntityII) = \Dsim^{*} \qty(\EntityII, \EntityI)$ for all $\EntityI, \EntityII \in \EntitySet$.
We can regard the setting discussed in this section as graph embedding if there exists a true undirected graph $\Graph = \qty(\VertexSet, \EdgeSet)$, where $\EdgeSet \subset \PairSet{\EntitySet}$, and the true dissimilarity is given by the distance function $\Distance_{\Graph}$ defined by the graph $\Graph$ and we set the threshold $\Threshold = 1.5$.
Here, $\Distance_{\Graph} \qty(\EntityI, \EntityII)$ is defined by the length of a shortest path in $\Graph$ between $\EntityI$ and $\EntityII$.
Note that $\Distance_{\Graph} \qty(\EntityI, \EntityI) = 0$ for all $\EntityI \in \VertexSet$, and $\Distance_{\Graph} \qty(\EntityI, \EntityI) = \infty$ if there exists no path between $\EntityI$ and $\EntityII$.
Here, $\Distance_{\Graph} \qty(\EntityI, \EntityII) < \Threshold = 1.5$ if and only if $\qty{\EntityI, \EntityII} \in \EdgeSet$ or $\EntityI = \EntityII$. 
Thus, we can regard graph embedding as a special case of the discussion here.

\subsection{Representation space and the objective of representation learning}
\label{sub:objective}
Fix some space $\ReprSpace$ with a \NewTerm{distance function} $\Distance_{\ReprSpace}: \ReprSpace \times \ReprSpace \to \Real$ that is symmetric, \IE $\Distance_{\ReprSpace} \qty(\Repr, \Repr') = \Distance_{\ReprSpace} \qty(\Repr', \Repr)$ for all $\Repr, \Repr' \in \ReprSpace$.
Here, we consider two points $\Repr, \Repr' \in \ReprSpace$ to be ``distant'' if $\Distance_{\ReprSpace} \qty(\Repr, \Repr')$ is large and ``close'' if $\Distance_{\ReprSpace} \qty(\Repr, \Repr')$ is small.
We call $\ReprSpace$ the \NewTerm{representation space}.
The most typical example is the $\NAxes$-dimensional Euclidean space $\qty(\Real^{\NAxes}, \Distance_{\Real^{\NAxes}})$, where $\Distance_{\Real^{\NAxes}}: \ReprSpace \times \ReprSpace \to \Real_{\ge 0}$ 
defined by $\Distance_{\Real^{\NAxes}} \qty(\PointVec, \PointVec') = \sqrt{\qty(\PointVec - \PointVec')^\Transpose \qty(\PointVec - \PointVec')}$.
Note that our main theorem allows distance functions not satisfying non-negativity or triangle inequality. See \Asp \ref{asp:Main} for the rigorous conditions.

The objective of representation learning is to get a map $\ReprMap: \EntitySet \to \ReprSpace$ which maps an entity to a representation in $\ReprSpace$, such that the representations are consistent to the true dissimilarity defined by $\Dsim^{*}$.
Here, we call $\ReprMap$ the \NewTerm{representation map}, and for $\Entity \in \EntitySet$, we call $\Repr_{\Entity} \DefEq \ReprMap \qty(\Entity) \in \ReprSpace$ the \NewTerm{representation} of entity $\Entity$.
Specifically, the objective of representation learning is to find a good representation map $\ReprMap$ that satisfies
\begin{equation}
  \label{eqn:ideal}
  \Dsim^{*} \qty(\EntityI, \EntityII) \lessgtr \Threshold \Leftrightarrow \Distance_{\ReprSpace} \qty(\Repr_{\EntityI}, \Repr_{\EntityII}) \lessgtr \Threshold_{\ReprSpace},
\end{equation}
for ``most'' $\qty{\EntityI, \EntityII} \in \PairSet{\EntitySet}$. We quantify the meaning of ``most'' in Section \ref{sub:risk}.
Here, $\Threshold_{\ReprSpace} \in \Real$ is a threshold value.
To make the formulation compatible with learning theory's notation, we rewrite the above representation learning objective as follows.
Define the true label function $\Label^{*}: \PairSet{\EntitySet} \to \qty{\pm 1}$ by $\Label^{*} \qty(\qty{\EntityI, \EntityII}) \DefEq \Sign \qty(\Dsim^{*} \qty(\EntityI, \EntityII) - \Threshold).$
Let $\SomeFunc: \Real \to \Real$ be a nondecreasing function and define the \NewTerm{hypothesis function} $\Hypothesis_{\ReprMap, \SomeFunc}: \PairSet{\EntitySet} \to \Real$ by 
\begin{equation}
\label{eqn:Hypothesis}
\Hypothesis_{\ReprMap, \SomeFunc} \qty(\qty{\EntityI, \EntityII}) \DefEq \SomeFunc \qty(\Distance_{\ReprSpace} \qty(\Repr_{\EntityI}, \Repr_{\EntityII})) - \SomeFunc \qty(\Threshold_{\ReprSpace}).
\end{equation}
Then, we can see that \eqref{eqn:ideal} is equivalent to the following.
\begin{equation}
\label{eqn:IdealHypothesis}
\Label^{*} \qty(\qty{\EntityI, \EntityII}) \Hypothesis_{\ReprMap, \SomeFunc} \qty(\qty{\EntityI, \EntityII}) > 0.
\end{equation}
Thus, our objective to find a representation map $\ReprMap$ that satisfies the above inequality for ``most'' $\qty{\EntityI, \EntityII} \in \PairSet{\EntitySet}$.

\subsection{Couple-label pair data and graph embedding}
\label{sub:data}
We have discussed the objective of representation learning in Section \ref{sub:objective}. 
To achieve the objective, we need to use some data that contain partial information about the true dissimilarity $\Dsim^{*}$.
For simple discussion, this paper focus on representation learning using couple-label pair data, which includes graph embedding as an important special case.
Still, our theory straightforwardly applies to existing settings, \IE that in \cite{suzuki2021graph} as we discuss in \Sub \ref{sec:proof}.

A couple-label pair data sequence is a sequence $\qty(\Point_{\IDatum})_{\IDatum=1}^{\NData}$ of pairs of an unordered entity couple and a label. 
Specifically, the $\IDatum$-th data point $\Point_{\IDatum} = \qty(\Feature_{\IDatum}, \Label_{\IDatum})$ consists of a pair of an unordered entity couple $\Feature_{\IDatum} = \qty{\EntityI_{\IDatum}, \EntityII_{\IDatum}} \in \PairSet{\EntitySet}$ and a label $\Label_{\IDatum} \in \qty{\pm 1}$.
Here, $\Label_{\IDatum} = +1$ indicates that the $\IDatum$-th data point claims $\EntityI_{\IDatum}$ and $\EntityII_{\IDatum}$ being similar, \IE $\Dsim^{*} \qty(\EntityI_{\IDatum}, \EntityII_{\IDatum}) < \Threshold$, and $\Label_{\IDatum} = -1$ indicates its converse, \IE $\Dsim^{*} \qty(\EntityI_{\IDatum}, \EntityII_{\IDatum}) > \Threshold$.
Nevertheless, this correspondence between the label $\Label_{\IDatum}$ and dissimilarity $\Dsim^{*} \qty(\EntityI_{\IDatum}, \EntityII_{\IDatum})$ does not always hold because the data point may be wrong owing to data noise.
As discussed in the previous subsection, if the true dissimilarity is given by the distance function $\Distance_{\Graph}$ defined by the graph $\Graph$ and we set the threshold $\Threshold = 1.5$, then $\Label_{\IDatum} = +1$ claims that there exists an edge between $\EntityI_{\IDatum}$ and $\EntityII_{\IDatum}$.

\subsection{Loss function}
To obtain representations using data as we discussed in Section \ref{sub:data}, we need to quantify how compatible representations are to the data.
This is what a \NewTerm{loss function} does.
This subsection defines the loss function for generic cases.
The definitions in the remainder of this subsection consider a generic prediction setting from a feature space $\FeatureSet$ to label space $\LabelSet$ to compare the couple-label pair case to the general discussion later.
Still, we can always specialize the discussion by substituting $\FeatureSet = \PairSet{\EntitySet}$ and $\LabelSet = \qty{\pm 1}$.
Given a loss function $\Loss: \FeatureSet \times \LabelSet \times \Real \to \Real_{\ge 0}$, the loss of a hypothesis function $\Hypothesis: \FeatureSet \to \Real$ on a data point $\qty(\Feature_{\IDatum}, \Label_{\IDatum}) \in \FeatureSet \times \LabelSet$ is given by $\Loss \qty(\Feature_{\IDatum}, \Label_{\IDatum}, \Hypothesis \qty(\Feature_{\IDatum}))$.
In the couple-label pair case, our main interest on a data point $\qty(\Feature_{\IDatum}, \Label_{\IDatum})$ whether the hypothesis function's output $\Hypothesis \qty(\Feature_{\IDatum})$ has the same sign as the label $\Label_{\IDatum}$ has, as discussed in \Sub \ref{sub:objective}.
That is, our interest is the sign of $\Label_{\IDatum} \Hypothesis \qty(\Feature_{\IDatum})$. 
Hence, we mainly consider a \NewTerm{margin-based loss}, that is, a loss function that can be defined by $\Loss \qty(\Feature, \Label, \Prediction) \DefEq \ReprFunc \qty(\Label \Prediction)$, where $\ReprFunc: \Real \to \Real_{\ge 0}$ is a non-increasing function.
Here, we assume $\ReprFunc$ is non-increasing because it is desirable and deserve a low loss if $\Label_{\IDatum} \Hypothesis \qty(\Feature_{\IDatum})$ is positive and vice versa.
The function $\ReprFunc$ is called a \NewTerm{representing function}.
A typical example is the \NewTerm{hinge loss function} defined by $\ReprFunc_\mathrm{hinge} \qty(\Prediction) \DefEq \max \qty{- \Prediction + 1, 0}$, which is non-increasing.

If the input of the loss function is unrestricted, the loss can be unbounded, which can lead to infinite risk.
Hence, we introduce \NewTerm{clipping} following \citep[Chapter 2,][]{steinwart2008support}.
For $\ClipBound \in \Real_{\ge 0}$, we define the \NewTerm{clipped value} $\SymClip{\Prediction}{\ClipBound} \in [-\ClipBound, +\ClipBound]$ by the median of the three element set $\qty{- \ClipBound, \Prediction, + \ClipBound}$.
Fix $\ClipBound \in \Real_{\ge 0}$, and we say that a loss function $\Loss: \FeatureSet \times \LabelSet \times \Real \to \Real_{\ge 0}$ is \NewTerm{clippable} at $\ClipBound$ if $\Loss \qty(\Feature, \Label, \SymClip{\Prediction}{\ClipBound}) \le \Loss \qty(\Feature, \Label, \Prediction)$ for all $\qty(\Feature, \Label) \in \FeatureSet \times \LabelSet$.
For example, the hinge loss function is a typical clippable loss.

\subsection{Data distribution and risks}
\label{sub:risk}
We assume that a data point is generated by a distribution $\Measure$ on $\FeatureSet \times \LabelSet$.
Once a distribution $\Measure$ is given, our interest is the expectation of the loss of a hypothesis function $\Hypothesis$ with respect to $\Measure$.
This expectation is called the \NewTerm{expected risk} of $\Hypothesis$ with respect to the loss function $\Loss$ and distribution $\Measure$, denoted by $\Risk_{\Loss, \Measure} \qty(\Hypothesis)$.
Here, the \NewTerm{risk function} $\Risk_{\Loss, \Measure}: \Real^{\FeatureSet} \to \Real_{\ge 0}$ is defined by 
\begin{equation}
\label{eqn:RiskDef}
\Risk_{\Loss, \Measure} \qty(\Hypothesis) \DefEq \Expect_{\qty(\Feature, \Label) \sim \Measure} \LossHypothesis_{\Loss, \Hypothesis} \qty(\Feature, \Label),
\end{equation}
where  $\LossHypothesis_{\Loss, \Hypothesis}: \FeatureSet \times \LabelSet \to \Real_{\ge 0}$ is defined by $\LossHypothesis_{\Loss, \Hypothesis} \qty(\Feature, \Label) \DefEq \Loss \qty(\Feature, \Label, \Hypothesis \qty(\Feature))$.
We also define the \NewTerm{clipped expected risk} $\SymClip{\Risk_{\Loss, \Measure}}{\ClipBound} \qty(\Hypothesis) \DefEq \Expect_{\qty(\Feature, \Label) \sim \Measure} \SymClip{\LossHypothesis_{\Loss, \Hypothesis}}{\ClipBound} \qty(\Feature, \Label)$,
where $\SymClip{\LossHypothesis_{\Loss, \Hypothesis}}{\ClipBound} \qty(\Feature, \Label) \DefEq \Loss \qty(\Feature, \Label, \SymClip{\Hypothesis \qty(\Feature)}{\ClipBound})$.
Now, we can formally state that the objective of representation learning is to minimize the expected risk.
Since the definition of the risk involves expectation, we only consider a measurable function as a hypothesis function, \IE $\Hypothesis \in \mathcal{L}_{0} \qty(\FeatureSet)$.
However, in the couple-pair label case, since $\FeatureSet$ is a finite set and we consider discrete topology, we have that $\mathcal{L}_{0} \qty(\FeatureSet) = 2^{\FeatureSet}$.
Thus, every function on $\FeatureSet$ is measurable and we can ignore the discussion on measurability.
Although the objective of representation learning is to minimize the expected risk, we cannot directly do that since we cannot directly observe the data distribution.
Instead, since we have a data sequence $\qty(\Feature_{\IDatum}, \Label_{\IDatum})_{\IDatum=1}^{\NData}$, we minimize the \NewTerm{empirical risk}
\begin{equation}
\Risk_{\Loss, \EmpiricalMeasure} \qty(\Hypothesis) = 
\Expect_{\qty(\Feature, \Label) \sim \EmpiricalMeasure} \LossHypothesis_{\Loss, \Hypothesis} \qty(\Feature, \Label)
= \frac{1}{\NData} \sum_{\IDatum=1}^{\NData} \LossHypothesis_{\Loss, \Hypothesis} \qty(\Feature_{\IDatum}, \Label_{\IDatum}).
\end{equation}
which is the risk calculated on the empirical measure $\EmpiricalMeasure: 2^{\FeatureSet \times \LabelSet} \to \Real$ defined by $\EmpiricalMeasure \DefEq \frac{1}{\NData} \sum_{\IDatum=1}^{\NData} \delta_\qty(\Feature_{\IDatum}, \Label_{\IDatum})$.
Here, $\delta_\qty(\Feature, \Label) \qty(\SetI) = 1$ if $\qty(\Feature, \Label) \in \SetI$ and $\delta_\qty(\Feature, \Label) \qty(\SetI) = 0$ otherwise.

Or, we might minimize the clipped version $\SymClip{\Risk_{\Loss, \EmpiricalMeasure}}{\ClipBound}$.
We remark that if the loss function $\Loss$ is clippable, then $\SymClip{\Risk_{\Loss, \EmpiricalMeasure}}{\ClipBound} \le \Risk_{\Loss, \EmpiricalMeasure} \qty(\Hypothesis)$ for all $\Hypothesis \in \mathcal{L}_{0} \qty(\FeatureSet)$.
Following \citep{steinwart2008support}, we define empirical risk minimization below so that the definition includes minimization of both versions.
\begin{definition}
Let $\HypothesisSet \subset \mathcal{L}_{0} \qty(\FeatureSet)$, and fix $\ERMMargin \in \Real_{\ge 0}$.
Then a map $\Learning: \qty(\FeatureSet \times \LabelSet)^{\NData} \to \mathcal{L}_{0} \qty(\FeatureSet)$ is called an $\ERMMargin$-\NewTerm{approximation clipped empirical risk minimization} ($\ERMMargin$-\NewTerm{CERM}) if it satisfies
\begin{equation}
\label{eqn:cerm}
\SymClip{\Risk}{\ClipBound}_{\Loss, \EmpiricalMeasure} \qty(\Learning \qty( \qty(\Point_{\IDatum})_{\IDatum=1}^{\NData}))
\le
\inf \qty{\Risk_{\Loss, \EmpiricalMeasure} \qty(\Hypothesis) \middle| \Hypothesis \in \HypothesisSet}
+ \ERMMargin,
\end{equation}
for all $\qty(\Point_{\IDatum})_{\IDatum=1}^{\NData} \in \qty(\FeatureSet \times \LabelSet)^{\NData}$ and empirical measure $\EmpiricalMeasure$ determined by $\qty(\Point_{\IDatum})_{\IDatum=1}^{\NData}$.
A $0$-CERM is called a \NewTerm{clipped empirical risk minimization} (\NewTerm{CERM}).
\end{definition}
\begin{remark}
The left hand side of the inequality in \eqref{eqn:cerm} is not clipped. 
Hence, if the loss function $\Loss$ is clippable and a map $\Learning$ minimizes either the non-clipped empirical risk or the clipped one, then it is a CERM.  
\end{remark}

Since we want to have as low a risk as possible, we are interested in the infimum of the risk.
We denote the infimum of the risk in a given hypothesis function set $\HypothesisSet \subset \mathcal{L}_{0} \qty(\FeatureSet)$ by $\SymClip{\Risk^{*, \HypothesisSet}_{\Loss, \Measure}}{\ClipBound}$, defined by 
$\SymClip{\Risk^{*, \HypothesisSet}_{\Loss, \Measure}}{\ClipBound}
\DefEq
\inf \qty{\SymClip{\Risk_{\Loss, \Measure}}{\ClipBound} \qty(\Hypothesis) \middle| \Hypothesis \in \HypothesisSet}$.

The infimum $\SymClip{\Risk^{*}_{\Loss, \Measure}}{\ClipBound} \DefEq \SymClip{\Risk^{*, \mathcal{L}_{0} \qty(\FeatureSet)}_{\Loss, \Measure}}{\ClipBound}$ of the expected risk over all  hypothesis functions is called the \NewTerm{Bayes risk}.

Since we try to achieve the Bayes risk using a CERM, we are interested in how well it goes. Hence, we will evaluate the \NewTerm{excess risk} defined by
$\SymClip{\Risk_{\Loss, \Measure}}{\ClipBound} \qty(\Hypothesis)
- \SymClip{\Risk^{*}_{\Loss, \Measure}}{\ClipBound}$, where $\Hypothesis$ is a CERM.
We can regard the excess risk of a CERM as a quantification of the generalization error.

\section{Fast rate of generalization error bound in representation learning}
\label{sec:Main}
This section states our upper bounds of the excess risk in representation learning on couple-label pair data.
\begin{assumption}
\label{asp:Main}
Fix $\ClipBound \in \Real_{> 0}$. Consider the following conditions regarding the representation space $\ReprSpace$, the dissimilarity function $\Dsim_{\ReprSpace}$, the function $\SomeFunc$, the loss function $\Loss$, and the data distribution $\Measure$.
\begin{enumerate}[label=(C\arabic*),topsep=0pt]
\item \label{item:iid} The random variables $\Point_{1}, \Point_{2}, \dots, \Point_{\NData}$ follow the distribution $\Measure$ on $\FeatureSet \times \LabelSet$ mutually independently.
\item \label{item:compact} The representation space $\ReprSpace$ is a topological space and \textbf{compact}.
\item \label{item:ContDis} The dissimilarity function $\Distance_{\ReprSpace}: \ReprSpace \times \ReprSpace \to \Real$ is a \textbf{continuous} symmetric function on $\ReprSpace \times \ReprSpace$.
\item \label{item:ContSome} The feature and label spaces are $\FeatureSet = \PairSet{\EntitySet}$ and $\LabelSet = \qty{\pm 1}$, and the hypothesis function set $\HypothesisSet$ is given by 
$\HypothesisSet = \HypothesisSet_{\ReprMap, \SomeFunc} \DefEq \qty{\Hypothesis_{\ReprMap, \SomeFunc} \middle| \ReprMap: \EntitySet \to \ReprSpace}$, where $\Hypothesis_{\ReprMap, \SomeFunc}$ is defined by \eqref{eqn:Hypothesis} and $\SomeFunc: \Real \to \Real$ is a \textbf{continuous} non-decreasing function.
\item \label{item:Clippable} The loss function $\Loss$ is clippable at $\ClipBound$.
\item \label{item:MarginBased} The loss function $\Loss$ is margin-based with a representing function $\ReprFunc: \Real \to \Real_{\ge 0}$.
\item \label{item:LossSup} The loss function $\Loss$ satisfies the \NewTerm{supremum bound condition}, \IE $\exists \LossBound \in \Real_{>0}, \forall \qty(\Feature, \Label) \in \FeatureSet \times \LabelSet, \forall \Prediction \in [- \ClipBound, \ClipBound]:  \Loss \qty(\Feature, \Label, \Prediction) \le \LossBound.$
\item \label{item:ReprLip} The representing function $\ReprFunc$ is Lipschitz continuous \IE there exists a constant $\LipBound \in \Real_{\ge 0}$ such that $\ReprFunc \qty(\Prediction - \Prediction') \le \LipBound \abs{\Prediction - \Prediction'}$ for any $\Prediction, \Prediction' \in [-\ClipBound, \ClipBound]$.
\item \label{item:BestHypothesis} There exists a \NewTerm{Bayes decision function} in $\HypothesisSet$, \IE there exists a hypothesis function $\Hypothesis^{*} \in \HypothesisSet$ that satisfies $\SymClip{\Risk_{\Loss, \Measure}}{\ClipBound} \qty(\Hypothesis^{*}) = \inf \qty{\SymClip{\Risk_{\Loss, \Measure}}{\ClipBound} \qty(\Hypothesis) \middle| \Hypothesis \in \mathcal{L}_{0} \qty(\FeatureSet)}$.
\item \label{item:LossVar} There exists $\VarExponent \in [0, 1]$ such that the \NewTerm{variance bound condition} holds, \IE there exists $\ConvexBound \in \Real_{\ge 0}$ such that for all $\Hypothesis \in \HypothesisSet$,
\begin{equation}
\begin{split}
    & \Expect_{\Feature \sim \Measure_{\FeatureSet}} \qty[\SymClip{\Hypothesis  \qty(\Feature)}{\ClipBound} - \SymClip{\Hypothesis^{*} \qty(\Feature)}{\ClipBound}]^{2}
    \le \ConvexBound \qty[\SymClip{\Risk_{\Loss, \Measure}}{\ClipBound} \qty(\Hypothesis) - \SymClip{\Risk_{\Loss, \Measure}}{\ClipBound} \qty(\Hypothesis^{*})]^\VarExponent.
\end{split}
\end{equation}
\end{enumerate}
\end{assumption}
\begin{remark}
In \Asp \ref{asp:Main},
\begin{enumerate}[label=(\alph*), leftmargin=*,topsep=0pt]
\item If $\ReprSpace$ is a metric space, then the condition \ref{item:compact} holds if and only if $\ReprSpace$ is totally bounded and complete. Also, its distance function always satisfy \ref{item:ContDis}.
For example, a closed ball in  finite-dimensional Euclidean or hyperbolic space can satisfy \labelcref{item:compact,item:ContDis}.
Furthermore, if $\ReprSpace$ is a subset of finite-dimensional Euclidean space, then \ref{item:compact} holds if and only if $\ReprSpace$ is bounded and closed.
If $\ReprSpace$ is a subset of finite-dimensional inner-product space, then the (negative) inner-product function is a continuous symmetric function, which satisfies \ref{item:ContDis} as a distance function.
\item In \ref{item:LossVar},
$\VarExponent = 1$ requires the ``strong-convexity'' of the loss function $\Loss$ with respect to the hypothesis function $\Hypothesis$, which is assumed in, \EG \citep{bartlett2005local, koltchinskii2006local}.
\item The conditions \ref{item:ReprLip} and \ref{item:LossVar} imply the following
\begin{equation}
\label{eqn:OriginalVarBound}
\begin{split}
    & \Expect_{\qty(\Feature, \Label) \sim \Measure} \qty[\SymClip{\LossHypothesis_{\Loss, \Hypothesis}}{\ClipBound} \qty(\Feature, \Label) - \SymClip{\LossHypothesis_{\Loss, \Hypothesis^{*}}}{\ClipBound} \qty(\Feature, \Label)]^{2} \\
    & \le \VarBound \qty[\Clip{\Risk_{\Loss, \Measure}}^{\ClipBound} \qty(\Hypothesis) - \Clip{\Risk_{\Loss, \Measure}}^{\ClipBound} \qty(\Hypothesis^{*})]^\VarExponent,
\end{split}
\end{equation}
for $\Hypothesis \in \HypothesisSet$
with $\VarBound = \LipConst^{2} \ConvexBound$, which corresponds to the condition assumed in \citep[Section 7,][]{steinwart2008support}.
\end{enumerate}
\end{remark}
We will discuss specific examples satisfying \Asp \ref{asp:Main} in \Sub \ref{sec:example}.
The following is our main result.
\begin{theorem}
\label{thm:Main}
Suppose that \labelcref{item:iid,item:compact,item:ContDis,item:ContSome,item:Clippable,item:MarginBased,item:ReprLip,item:LossSup} in \Asp \ref{asp:Main} holds.
Let $\LipBound, \LossBound \in \Real_{\ge 0}$ be constants that satisfy the inequalities in \cref{item:ReprLip,item:LossSup} in \Asp \ref{asp:Main}, respectively, and define $\FuncLLBound \in \Real_{\ge 0}$ by
\begin{equation}
\begin{split}
\label{eqn:FuncLLBound}
\FuncLLBound^{2} 
& =
\max_{\Hypothesis \in \HypothesisSet} \sum_{\qty{\EntityI, \EntityII} \in \PairSet{\EntitySet}} \qty[\Hypothesis \qty(\qty{\EntityI, \EntityII})]^{2}.
\end{split}
\end{equation}
Fix $\ExceptProbability \in \Real_{> 0}$ and $\ERMMargin \in \Real_{\ge 0}$.
Then 

(i) There exists a measurable (0-)CERM.

(ii) Any $\ERMMargin$-CERM $\Learning: \qty(\FeatureSet \times \LabelSet)^{\NData} \to \Real$ satisfies
\begin{equation}
\begin{split}
& \SymClip{\Risk_{\Loss, \Measure}}{\ClipBound} \qty(\Learning \qty(\qty(\Point_{\IDatum})_{\IDatum=1}^{\NData})) - \inf \qty{\Risk_{\Loss, \Measure} \qty(\Hypothesis) \middle| \Hypothesis \in \HypothesisSet} 
\\
& \le \Rate_{0} \qty(\NData) + \FastMinorRate' \qty(\NData)
+ \ERMMargin,
\end{split}
\end{equation}
in probability at least $1 - \ExceptProbability$, where
\begin{equation}
\begin{split}
\Rate_{0} \qty(\NData)
\DefEq
4 \LipBound \FuncLLBound \cdot \qty(\frac{2}{\NData})^{\frac{1}{2}},
\FastMinorRate' \qty(\NData)
\DefEq
\LossBound_{0} \cdot \qty(\frac{\ln \frac{1}{\ExceptProbability}}{\NData})^\frac{1}{2}.
\end{split}
\end{equation}

(iii) In addition, suppose \cref{item:BestHypothesis,item:LossVar} hold, and let $\ConvexBound \in \Real_{\ge 0}$ and $\VarExponent \in [0, 1]$ be constants that satisfy the inequalities in \cref{item:LossVar} of \Asp \ref{asp:Main} and fix $\LossBound_{0} > \LossBound$.
Then every $\ERMMargin$-CERM $\Learning: \qty(\FeatureSet \times \LabelSet)^{\NData} \to \Real$ satisfies
\begin{equation}
\begin{split}
\label{eqn:LocalRate}
& \SymClip{\Risk_{\Loss, \Measure}}{\ClipBound} \qty(\Learning \qty(\qty(\Point_{\IDatum})_{\IDatum=1}^{\NData})) - \SymClip{\Risk_{\Loss, \Measure}^{*}}{\ClipBound} \\
& \le \min \qty{\Rate_{\ICouple} \qty(\NData) \middle| 0 \le \ICouple \le \PairSetCardinal{\EntitySet}} \lor \SlowMinorRate \qty(\NData) \lor \FastMinorRate \qty(\NData)
+ 3 \ERMMargin,
\end{split}
\end{equation}
in probability at least $1 - \ExceptProbability$. 
Here,
$\SlowMinorRate, \FastMinorRate: \Integer_{> 0} \to \Real_{\ge 0}$ are defined by
\begin{equation}
\small
\begin{split}
\SlowMinorRate \qty(\NData)
\DefEq
3 \cdot \qty(\frac{72 \qty(\LossBound^{2 - \VarExponent} \lor \LipBound^2 \ConvexBound) \ln \frac{3}{\ExceptProbability}}{\NData})^{\frac{1}{2 - \VarExponent}},
\FastMinorRate \qty(\NData)
\DefEq
\frac{15 \LossBound_{0}}{\NData} \ln \frac{3}{\ExceptProbability},
\end{split}
\end{equation}
and, $\Rate_{\ICouple} \qty(\NData)$ for $1, \dots \PairSetCardinal{\EntitySet}$ is defined as the only positive solution of the equation $\Rate \qty(\NData) = \nicefrac{30 \UnRdmFunc_{\ICouple} \qty(\Rate)}{\sqrt{\NData}}$ for $\Rate$, where 
\vspace{-5pt}
\begin{equation}
\label{eqn:UnRdmFunc}
\UnRdmFunc_{\ICouple} \qty(\Rate) 
\DefEq
2 \LipBound \sqrt{2 \FuncLLBound^{2} \qty(\frac{\ConvexBound \Rate^{\VarExponent}}{4 \FuncLLBound^{2}} \ICouple + \SumProbability_{\Measure_{\FeatureSet}} \qty(\PairSetCardinal{\EntitySet} - \ICouple))},
\vspace{-5pt}
\end{equation}
with $\SumProbability_{\Measure_{\FeatureSet}} \qty(\ICoupleII) \DefEq \min_{\SubCoupleSet \subset \PairSet{\EntitySet}, \abs{\SubCoupleSet} = \ICoupleII} \Measure_{\FeatureSet} \qty(\SubCoupleSet)$.
In particular,
\vspace{-10pt}
\begin{equation}
\Rate_{\PairSetCardinal{\EntitySet}} \qty(\NData)
\DefEq
3 \cdot \qty(1800 \cdot \PairSetCardinal{\EntitySet} \cdot \frac{\LipBound^{2} \ConvexBound}{\NData})^{\frac{1}{2 - \VarExponent}}.
\vspace{-5pt}
\end{equation}
We define $\Rate_{\ICouple} \qty(\NData) = 0$ if $\LipBound \FuncLLBound \ConvexBound = 0$.
\end{theorem}
\begin{remark}
In \Thm \ref{thm:Main},
\begin{enumerate}[label=(\alph*),leftmargin=*,topsep=0pt]
\item Although \Asp \ref{asp:Main} does not explicitly assume the finiteness of $\FuncLLBound$, it follows the conditions \ref{item:compact}, \ref{item:ContDis}, and \ref{item:ContSome} since  $\FuncLLBound$ is defined as the maximum of a continuous function from a compact space $\ReprSpace^{\abs{\EntitySet}}$.
\item The bound \eqref{eqn:LocalRate} is $O \qty(\qty(\frac{1}{\NData})^{\frac{1}{\qty(2 - \VarExponent)}})$, which is faster than $O \qty(\qty(\frac{1}{\NData})^{\frac{1}{2}})$ if $\VarExponent > 0$. In particular, it is $O \qty(\frac{1}{\NData})$ if $\VarExponent = 0$.
\item We have that $\min \qty{\Rate_{\ICouple} \qty(\NData) \middle| 0 \le \ICouple \le \PairSetCardinal{\EntitySet}} \le \Rate_{0} \qty(\NData) \land \Rate_{\PairSetCardinal{\EntitySet}} \qty(\NData)$, whose right hand side is always analytically obtained.
Here, $\Rate_{0} \qty(\NData) \gtreqless \Rate_{\PairSetCardinal{\EntitySet}} \qty(\NData) 
 \Leftrightarrow 
\NData \gtreqless 7200 \LipBound^{2} \FuncLLBound^{2} \qty(\frac{\ConvexBound \PairSetCardinal{\EntitySet}}{4 \FuncLLBound^{2}})^{\frac{2}{\VarExponent}}$
if $\VarExponent > 0$.
This implies that the additional conditions in (iii) provides a faster rate for large $\NData$ if $\VarExponent > 0$.
Note that we can ignore $\SlowMinorRate$ and $\FastMinorRate$ unless we consider exponentially small $\ExceptProbability$. 
It is because $\SlowMinorRate$ and $\FastMinorRate$ are in no slower order with respect to $\NData$ than $\Rate_{\PairSetCardinal{\EntitySet}}$ and $\Rate_{0}$, respectively, and $\Rate_{\PairSetCardinal{\EntitySet}}$ and $\Rate_{0}$ depend on $\abs{\EntitySet}$ and $\FuncLLBound$ while $\SlowMinorRate$ and $\FastMinorRate$ are independent of them. 
\item As we can see from the definition of $\UnRdmFunc_{\ICouple}$, the behavior of $\Rate_{\ICouple} \qty(\NData)$ for $\ICouple = 1, 2, \dots, \PairSetCardinal{\EntitySet} - 1$ depend on $\Measure_{\FeatureSet}$.
If $\Measure_{\FeatureSet}$ is the uniform distribution on $\FeatureSet$, then $\Rate_{\ICouple} \qty(\NData) \ge \Rate_{0} \land \Rate_{\PairSetCardinal{\EntitySet}}$ for $\ICouple = 1, 2, \dots, \PairSetCardinal{\EntitySet} - 1$. 
Hence, we cannot improve the bound from $\Rate_{0} \qty(\NData) \land \Rate_{\PairSetCardinal{\EntitySet}} \qty(\NData)$.
As an extreme example of the other direction, consider the case where there exists some  $\SubCoupleSet \subset \PairSet{\EntitySet}$ satisfies $\Measure_{\FeatureSet} \qty(\SubCoupleSet) = 1$. 
Then we have that $\Rate_{\abs{\SubCoupleSet}}
\DefEq
3 \cdot \qty(1800 \cdot \abs{\SubCoupleSet} \cdot \frac{\LipBound^{2} \ConvexBound}{\NData})^{\frac{1}{2 - \VarExponent}}$.
This is given by replacing $\PairSetCardinal{\EntitySet}$ in $\Rate_{\PairSetCardinal{\EntitySet}}$ with $\abs{\SubCoupleSet}$.
In particular, $\Rate_{\abs{\SubCoupleSet}}
\le
\Rate_{\PairSetCardinal{\EntitySet}}$ and the equality holds if and only if $\SubCoupleSet = \PairSet{\EntitySet}$. 
This result is natural since $\Measure_{\FeatureSet} \qty(\SubCoupleSet) = 1$ means that we can ignore $\PairSet{\EntitySet} \setminus \SubCoupleSet$.
\end{enumerate}
\end{remark}
Specific advantages of \Thm \ref{thm:Main} over existing results will be discussed in \Sub \ref{sec:comparison}.
\section{Examples}
\label{sec:example}
\Asp \ref{asp:Main} and \Thm \ref{thm:Main} are given in a general form, including many parameters such as $\ConvexBound, \NoiseExponent$, and $\FuncLLBound$, which depend on the situation.
This section gives specific examples of calculating these values in some application cases, and a comparison between Euclidean and hyperbolic spaces using the calculations.

\subsection{Representation space and $\FuncLLBound$}
We assume that $\SomeFunc \qty(\Prediction) = \Prediction^\DistanceExponent$ for $\DistanceExponent \ge 1$, as a simplest case.
If $\ReprSpace$ is a metric space, whose radius is $\Radius$, then we have that $\FuncLLBound^{2} \le \PairSetCardinal{\EntitySet} \qty(\qty(2 \Radius)^{\DistanceExponent} - \qty(\Threshold_{\ReprSpace})^{\DistanceExponent})^{2} \le \PairSetCardinal{\EntitySet} \qty(2 \Radius)^{2 \DistanceExponent}$ if $\Threshold_{\ReprSpace} \in [0, 2 \Radius]$.
This is the worst case, and we have the following better bound for Euclidean space.
\begin{lemma}
\label{lem:EuclideanBound}
If $\ReprSpace$ is a subset of a closed ball with radius $2 \Radius$ in Euclidean space, then $\FuncLLBound^{2} \le \frac{\abs{\EntitySet}}{8}  \qty(2 \Radius)^{2 \DistanceExponent}$.
\end{lemma}
Here, the right side is linear for $\abs{\EntitySet}$.
On the other hand, the following lemma states that hyperbolic space almost achieves the worst case if the diameter is sufficiently large.
\begin{lemma}
\label{lem:HyperbolicBound}
If $\ReprSpace$ is a closed ball of radius $\Radius$ in hyperbolic space (dimension $\NAxes \ge 2$), then $\frac{\FuncLLBound^{2}}{\qty(2 \Radius)^{2 \DistanceExponent}} \to \PairSetCardinal{\EntitySet}$.
\end{lemma}
The above result, at one glance, suggests that a Euclidean ball is better than a hyperbolic ball. 
However, the discussion is not trivial since hyperbolic space usually has a better approximation error.
We will compare in \Sub \ref{sec:comparison} a Euclidean ball and a hyperbolic ball, considering both the approximation and generalization errors.

\subsection{Hinge loss and $\VarExponent$}
The upper bound in \Thm \ref{thm:Main} heavily depends on $\VarExponent$.
The value $\VarExponent$ is determined by $\Measure$ and $\Loss$, but its calculation is not trivial.
As an example case, we introduce the hinge loss case since it has been widely used as the loss function of the support vector machine \citep{cortes1995support} and mainly discussed in the context of generalization error analysis in the classification problem \citep{steinwart2008support, DBLP:conf/nips/JainJN16, gao2018importance, suzuki2021generalization, suzuki2021graph}.
Suppose the loss function is the hinge loss, \IE $\ReprFunc \qty(\Prediction) = \max \qty{- \Prediction + 1, 0}$.
Then, it is known that the parameter $\VarExponent$ of the variance bound condition \ref{item:LossVar} in \Asp \ref{asp:Main} depends on the data distribution.  
Define $\Posterior: \FeatureSet \to [0, 1]$ by $\Posterior \qty(\Feature)
\DefEq \frac{\Measure \qty(\qty{\qty(\Feature, +1)})}{\Measure_{\FeatureSet} \qty(\qty{\Feature})}$.
Note that we can ignore the definition for $\Feature$ such that $\Measure_{\FeatureSet} \qty(\qty{\Feature}) = 0$ since it is about a measure-zero space.
We say that a distribution $\Measure$ on $\Feature \times \qty{\pm 1}$ has \NewTerm{noise exponent} $\NoiseExponent \in \Real_{\ge 0}$ with constant $\NoiseConst \in \Real_{>0}$ if 
$\Measure_{\FeatureSet} \qty(\qty{\Feature \in \FeatureSet \middle| \abs{2 \Posterior \qty(\Feature) - 1} < t}) \le \qty(\NoiseConst t)^{\NoiseExponent}$,
and noise exponent $+ \infty$ with constant $\NoiseConst \in \Real_{>0}$ if
$\Measure_{\FeatureSet} \qty(\qty{\Feature \in \FeatureSet \middle| \abs{2 \Posterior \qty(\Feature) - 1} < \nicefrac{3}{\NoiseConst}}) = 0$,
where $\Measure_{\FeatureSet}$ is the marginal distribution of $\Measure$ on $\FeatureSet$ defined by $\Measure_{\FeatureSet} \qty(\SetI) \DefEq \Measure \qty(\SetI \times \qty{\pm 1})$ for a measurable set $\SetI \subset \FeatureSet$.
Here, a large $\NoiseExponent$ indicates a small noise.
The condition $\NoiseExponent = \infty$ corresponds to the strong low-noise condition, which has been assumed in, \EG \citep{koltchinskii2005exponential}. 
We have the following, using existing results about $\VarExponent$ for the hinge loss \citep[\EG Chapter 8,][]{steinwart2008support}.
\begin{corollary}
\label{cor:hinge}
Suppose that conditions \labelcref{item:iid,item:compact,item:ContDis,item:ContSome} in \Asp \ref{asp:Main} are satisfied and the loss function is the hinge loss given by  $\ReprFunc \qty(\Prediction') \DefEq \max \qty{1 - \Prediction', 0}$. 
Define $\FuncLLBound$ by \eqref{eqn:FuncLLBound} and let $\ClipBound=1$.
Fix $\ExceptProbability \in \Real_{> 0}$ and $\ERMMargin \in \Real_{\ge 0}$.
Then (i) there exists a measurable (0-)CERM.
(ii) With $\LipConst = 1$, (ii) of \Thm \ref{thm:Main} holds.
(iii) In addition, if the condition \ref{item:BestHypothesis} holds and the distribution $\Measure$ has noise exponent $\NoiseExponent \in \Real_{\ge 0}$ with constant $\NoiseConst \in \Real_{>0}$, then, with $\LossBound_{0} > 2$, $\VarExponent = \frac{\NoiseExponent}{\NoiseExponent + 1}$, and $\ConvexBound = 6 \NoiseConst^{\frac{\NoiseExponent}{\NoiseExponent + 1}}$, (iii) of \Thm \ref{thm:Main} holds.
\end{corollary}

\subsection{Improved comparison: Euclidean vs hyperbolic.}
\citet{suzuki2021graph} showed a sufficient condition for $\NData$ for graph embedding in hyperbolic space to be better than that in Euclidean space.
Following their paper's setting, we give a sufficient condition based on \Thm \ref{thm:Main}.
\newcommand{\NoiseMargin}{\xi}
\newcommand{\NViolate}{v}
\newcommand{\MinSymb}{\textrm{min}}
\newcommand{\Ball}{\mathcal{B}}

Assume that the posterior distribution is given by
$\Posterior \qty(\qty{\EntityI, \EntityII}) = \frac{1}{2} \qty(1 + \NoiseMargin \Label^{*} \qty(\qty{\EntityI, \EntityII}))$, where $\NoiseMargin \in [0, 1]$.
Fix $\SomeFunc$.
The hypothesis function $\Hypothesis_{\ReprMap, \SomeFunc}$ given by a representation map $\ReprMap$ gives a Bayes decision function if and only if
\begin{equation}
\label{eqn:MarginCondition}
\small
\begin{split}
  \Label^{*} \qty(\qty{\EntityI, \EntityII}) = +1 
  & \Rightarrow 
  \SomeFunc \qty(\Dsim_{\ReprSpace} \qty(\Repr_{\EntityI}, \Repr_{\EntityII})) \le \SomeFunc \qty(\Threshold_{\ReprSpace}) - 1,
  \\
  \Label^{*} \qty(\qty{\EntityI, \EntityII}) = -1 
  & \Rightarrow \SomeFunc \qty(\Dsim_{\ReprSpace} \qty(\Repr_{\EntityI}, \Repr_{\EntityII})) \ge \SomeFunc \qty(\Threshold_{\ReprSpace}) + 1.
\end{split}
\end{equation}
Note that the above condition is stronger than \eqref{eqn:ideal}.

For a representation space $\ReprSpace$ and a representation map $\ReprMap$, we define $\NViolate_{\ReprSpace}: \ReprSpace^\EntitySet \to \Integer_{\ge 0}$ by $\NViolate_{\ReprSpace} \qty(\ReprMap) \DefEq \abs{\qty{\SubCoupleSet \middle| \forall \qty{\EntityI, \EntityII} \in \SubCoupleSet: \textrm{$\qty{\EntityI, \EntityII}$ violates \eqref{eqn:MarginCondition}}}}$ and $\NViolate_{\MinSymb} \qty(\ReprSpace) \DefEq \min \qty{\NViolate_{\ReprSpace} \qty(\ReprMap) \middle| \ReprMap: \ReprSpace \to \EntitySet}$.
Then, we can see that $\Risk_{\Loss, \Measure}^{*, \HypothesisSet_{\ReprSpace, \SomeFunc}} - \Risk_{\Loss, \Measure}^{*} = \frac{\NViolate_{\MinSymb} \qty(\ReprSpace)}{\PairSetCardinal{\EntitySet}}\NoiseMargin$.
If $\ReprSpace$ is a metric space, let a closed ball with radius $\Radius$ in $\ReprSpace$ denoted by $\Ball [\Radius; \ReprSpace]$.

If the true dissimilarity $\Dsim^{*}$ is the graph distance of a tree.
The following lemmata regarding $\NViolate_{\MinSymb} \qty(\Ball [\Radius; \Euclid^{2}])$ and $\NViolate_{\MinSymb} \qty(\Ball [\Radius; \Hyperbolic^{2}])$ hold as straightforward modifications of results in
\citep{DBLP:conf/gd/Sarkar11, DBLP:conf/acml/SuzukiWTNY19, suzuki2021generalization}. 
(See the supplementary materials for the proofs).
\begin{lemma}
\label{lem:MarginEmbedding}
Suppose that $\qty(\VertexSet, \EdgeSet)$ is a tree and $\Dsim^{*}: \VertexSet \times \VertexSet \to \Real_{\ge 0}$ is given by its graph distance. 
Then, there exist $\Radius \in \Real_{\ge 0}$ such that $\NViolate_{\MinSymb} \qty(\Ball [\Radius; \Hyperbolic^{\NAxes}]) = 0$ for any $\NAxes$.
\end{lemma}

\newcommand{\Packing}{p}
\begin{lemma}
\label{lem:LGEDifficulty}
Let $\Packing \qty(\NAxes)$ be the packing number of the $\NAxes$-dimensional unit sphere with the unit distance.
In particular, $\Packing \qty(2) = 5$.
Suppose that the true dissimilarity $\Dsim^{*}$ is given by the graph distance of a graph $\Graph = \qty(\EntitySet, \EdgeSet)$.
Then, $\NViolate_{\MinSymb} \qty[\Radius; \Euclid^{2}]$ is larger than or equal to the number of disjoint $\qty(\Packing \qty(\NAxes) + 1)$-star subgraphs in the graph.
\end{lemma}
The above lemmata helps the comparison between embedding in Euclidean space and hyperbolic space. 
The following is an example of a specific comparison in the setting discussed in \citep{suzuki2021graph}.
For a more general discussion, see Appendix \ref{sec:GeneralCondition}.

\newcommand{\NChildren}{\lambda}
\newcommand{\Height}{h}
\begin{example}
\label{exm:Numerical}
We consider the complete balanced $\NChildren$-ary tree with height $\Height$, and the noise margin $\NoiseMargin = \frac{1}{2}$.
Suppose $\NChildren=5$ and $\Height=4$. 
Here, we have that $\abs{\EntitySet} = 156$ and \Lem \ref{lem:MarginEmbedding} gives $\Radius = 39.51$.
If $\NData \ge 1.19 \times 10^{9}$ for $\DistanceExponent=1$ or $\NData \ge 7.43 \times 10^{12}$ for $\DistanceExponent=2$, then in probability at least $1 - 2^{-10}$, the expected risk of a CERM using $\Ball [\Radius; \Hyperbolic^{2}]$ is better than that using any ball in $\Euclid^{2}$.
\end{example}
\begin{remark}
The above evaluation uses the approximation error of the embedding using $\Euclid^{2}$ as the lower bound of the error by ERM.
We may obtain a better threshold in the near future once we obtain a good lower bound of the generalization error of representation learning using $\Euclid^{2}$.
\end{remark}

\section{Core evaluation: Rademacher complexity}
\label{sec:proof}
In this section, we provide the core technical result used to prove \Thm \ref{thm:Main}, to make an essential comparison in \Sec \ref{sec:comparison} between our results and existing results, without being influenced by the loss function's non-essential difference.

Our proof depends on the standard schemes in the statistical learning theory using the \NewTerm{Rademacher complexity (RC)}.
\begin{definition}
Let  $\Rdm_{1}, \Rdm_{2}, \dots, \Rdm_{\NData}, \Point_{1}, \Point_{2}, \dots, \Point_{\NData}$ be mutually independent random variables, where each of $\Rdm_{1}, \Rdm_{2}, \dots, \Rdm_{\NData}$ takes values $\qty{-1, +1}$ with equal probability and each of $\Point_{1}, \Point_{2}, \dots, \Point_{\NData}$ follows some distribution $\Measure$ on a set $\PointSet$. 
The Rademacher complexity (RC) $\RdmCmpl_{\Measure, \NData} \qty(\HypothesisSet)$ of a function set $\HypothesisSet \subset \mathcal{L}_{0} \qty(\PointSet)$ on $\Measure$ is defined by
\vspace{-0pt}
\begin{equation}
\small
\RdmCmpl_{\Measure, \NData} \qty(\HypothesisSet)
\DefEq
\Expect_{\qty(\Point_{\IDatum})_{\IDatum=1}^{\NData}} \Expect_{\qty(\Rdm_{\IDatum})_{\IDatum=1}^{\NData}}
    \qty[\frac{1}{\NData} \sup_{\FuncI \in \FuncSetI} \sum_{\IDatum=1}^{\NData} \Rdm_{\IDatum} \FuncI \qty(\Point_{\IDatum})].
\end{equation}
\vspace{-10pt}
\end{definition}
In the following, we fix a measurable loss function $\Loss: \FeatureSet \times \LabelSet \times \Real \to \Real_{\ge 0}$ and hypothesis function set $\HypothesisSet \subset \mathcal{L}_{0} \qty(\FeatureSet)$, and we define $\LossHypothesis_{\Loss, \HypothesisSet} \subset \mathcal{L}_{0} \qty(\FeatureSet \times \LabelSet)$ by $\LossHypothesis_{\Loss, \HypothesisSet} \DefEq \qty{\LossHypothesis_{\Loss, \Hypothesis} \middle| \Hypothesis \in \HypothesisSet}$ and local hypothesis function set
$\HypothesisSet_{\Rate} \DefEq \qty{\Hypothesis \in \HypothesisSet \middle| \SymClip{\Risk_{\Loss, \Measure}}{\ClipBound} \qty(\Hypothesis) - \SymClip{\Risk_{\Loss, \Measure}^{*}}{\ClipBound} \le \Rate}$ for $\Rate \in \Real_{\ge 0}$.
Existing research, \citep[\EG][]{DBLP:journals/jmlr/BartlettM02}, has shown that we can obtain an upper bound of generalization error proportional to $\RdmCmpl_{\Measure, \NData} \qty(\LossHypothesis_{\Loss, \HypothesisSet})$.
It is also shown \citep[\EG][]{bartlett2005local, koltchinskii2006local} that we can obtain a faster upper bound by evaluating $\RdmCmpl_{\Measure, \NData} \qty(\LossHypothesis_{\Loss, \HypothesisSet_{\Rate}})$, which we call the \NewTerm{local Rademacher complexity (LRC)}. 
For the above reason, we are interested in the RC and LRC.
Our evaluation of the RC and LRC in the couple-label data learning setting is the following.
\begin{theorem}
\label{thm:RdmRepr}
Assume that the conditions \labelcref{item:compact,item:ContDis,item:ContSome,item:MarginBased,item:ReprLip} hold. Then, we have that
\begin{equation}
\begin{split}
\RdmCmpl_{\Measure, \NData} \qty(\LossHypothesis_{\Loss, \HypothesisSet_{\Rate}})
\le 
\min_{\ICouple = 0, 1, \dots, \PairSetCardinal{\EntitySet}} \frac{\UnRdmFunc_{\ICouple}\qty(\Rate)}{\sqrt{\NData}},
\end{split}    
\end{equation}
where $\UnRdmFunc_{\ICouple}$ is defined by \eqref{eqn:UnRdmFunc}.
In particular, by substituting $\Rate = + \infty$, we have that
$\RdmCmpl_{\Measure, \NData} \qty(\LossHypothesis_{\Loss, \HypothesisSet}) = \RdmCmpl_{\Measure, \NData} \qty(\LossHypothesis_{\Loss, \HypothesisSet_{\infty}}) = 2 \LipBound \FuncLLBound \cdot \qty(\frac{2}{\NData})^{\frac{1}{2}}$.
\end{theorem}
\begin{remark} Regarding \Thm \ref{thm:RdmRepr},
\begin{enumerate}[label=(\alph*),leftmargin=*,topsep=0pt]
    \item To the best of our knowledge, \Thm \ref{thm:RdmRepr} is the first LRC evaluation in the context of representation learning, including couple-label pair data learning and graph embedding. 
    \item \Thm \ref{thm:RdmRepr} implies that we can have a meaningful LRC evaluation even without regularization, though we need it for \EG the support vector machine analysis \citep{steinwart2008support}. It is advantageous since it can exploit the non-Euclidean space's representability in the resulting upper bound.    
    \item We can straightforwardly update existing generalization error bounds of graph embedding based on the RC, such as that in \citep{suzuki2021graph}, using the above bound, although this paper's discussion focuses on our simplest couple-label pair setting to save space. 
\end{enumerate}
\end{remark}

As explained, the RC evaluation is substantial in deriving generalization error bounds, regardless of the specific form of the loss function.
The discussion makes us ready for comparison in \Sec \ref{sec:comparison} between existing results and ours.

\section{Related work and comparison}
\label{sec:comparison}
The generalization error of representation learning has been studied for the ordinal data case \citep{DBLP:conf/nips/JainJN16, suzuki2021generalization} and where random variables associated with entities are observed \citep{wang2018erm}.
Still, the first paper that has derived a generalization error bound for a typical graph embedding setting is \citep{gao2018importance}, although this paper only considers linear space and gives a result with some unevaluated term.
To the best of our knowledge, only \citep{suzuki2021graph} considers the generalization error for graph embedding in non-Euclidean space, including hyperbolic space.
This section aims to compare our result with the result by \citep{suzuki2021graph}.
In \citep{suzuki2021graph}, the positive-negative example data case is mainly discussed, which needs a large space to introduce and has a loss function different from ours.
However, since the core technique of their result is also the RC evaluation, we can make an essential comparison between them throughout the evaluations.
The following is the result by \citep{suzuki2021graph}.

\begin{corollary}[Rademacher complexity evaluation by \citep{suzuki2021graph}]
\label{cor:OldBound}
Let $\ReprSpace$ be a closed ball with radius $\Radius$ in $\NAxes$-dimensional Euclidean space $\Euclid^{\NAxes}$ or hyperbolic space $\Hyperbolic^{\NAxes}$.
Let $\SomeFunc \qty(\Prediction) = \SomeFuncII \qty(\Prediction^{2})$ for Euclidean case and $\SomeFunc \qty(\Prediction) = \SomeFuncII \qty(\cosh \Prediction)$ for hyperbolic space case, where $\SomeFuncII: \Real_{\ge 0} \to \Real$ is a non-descreasing Lipschitz continuous function whose Lipschitz constant is $\LipConst_{\SomeFuncII}$.
Also, let the loss function be the hinge loss given by $\ReprFunc \qty(\Prediction') \DefEq \max \qty{1 - \Prediction', 0}$. Then
\begin{equation}
\begin{split}
\RdmCmpl_{\Measure, \NData} \qty(\LossHypothesis_{\Loss, \HypothesisSet_{\ReprSpace, \SomeFunc}})
\le \frac{\omega \qty(\Radius)}{\NData} \LipConst_{\SomeFuncII} \abs{\EntitySet} \qty(\sqrt{2 \NData \EdgeMatVarNorm \ln \abs{\EntitySet}} + \frac{\EdgeMatSV}{3} \ln \abs{\EntitySet}),
\end{split}    
\end{equation}
where $\omega \qty(\Radius) \DefEq \qty(2 \Radius)^2$ and $\EdgeMatSV = 2$ for Euclidean ball cases, and $\omega \qty(\Radius) \DefEq \cosh^2 \Radius + \sinh^2 \Radius$ and $\EdgeMatSV = \frac{1}{2}$ for hyperbolic ball cases. See Appendix \ref{sec:OldDetail} for the definition of $\EdgeMatVarNorm$, which depends on $\Measure_{\FeatureSet}$ and $\abs{\EntitySet}$.
\end{corollary}
\begin{remark}[Comparison of \Thm \ref{thm:RdmRepr} to \Cor \ref{cor:OldBound}]
\label{rem:comparison}
\leavevmode
\begin{enumerate}[label=(\alph*), leftmargin=*,topsep=0pt]
\item \Thm \ref{thm:RdmRepr} can apply to the most natural case $\SomeFunc \qty(\Prediction) = \Prediction$, while \Cor \ref{cor:OldBound} cannot since $\SomeFuncII \qty(\Prediction) = \sqrt{\Prediction}$ or $\SomeFuncII \qty(\Prediction) = \Acosh \Prediction$ is not Lipschitz continuous.
\item No LRC evaluation in \cite{suzuki2021graph}. 
Hence we cannot derive a faster bound than $O(\frac{1}{\sqrt{\NData}})$ in their direction, while we did as in (iii) of \Thm \ref{thm:Main} thanks to the LRC evaluation by \Thm \ref{thm:RdmRepr}.
\item The bound in \Thm \ref{thm:RdmRepr} is polynomial in $\Radius$ even for hyperbolic space, better than \Cor \ref{cor:OldBound}, which is exponential in $\Radius$. The comparison regarding the dependency on $\abs{\EntitySet}$ is complicated. If we regard other variables as constants, \Thm \ref{thm:RdmRepr}, which is $O \qty(\abs{\EntitySet})$, is always better than \Cor \ref{cor:OldBound} owing to the second term in \Cor \ref{cor:OldBound}. However, if $\NData$ is sufficiently large, then the second term vanishes. In that case, the discussion depends on $\EdgeMatVarNorm$, which again depends on $\Measure_{\FeatureSet}$. See Appendix \ref{sec:OldDetail} for detailed discussion. In any case, the bound in \Thm \ref{thm:Main} is much better in practical evaluations as the following example shows owing to the difference in the dependency on $\Radius$.
\item For \Exm \ref{exm:Numerical} with $\SomeFunc \qty(\Prediction) = \Prediction$, \Cor \ref{cor:OldBound} gives $\NData \ge 7.30 \times 10^{72}$, a much larger data size than that by \Thm \ref{thm:Main}, as a sufficient condition for the hyperbolic method to outperform Euclidean method.
\end{enumerate}
\end{remark}

\section{Discussion on proof and future work}
\label{sec:FutureWork}

As we explained in the Introduction section, our idea is to regard each hypothesis function as a function of the $\PairSetCardinal{\EntitySet}$ distance values, each of which corresponds to a couple of entities.
Specifically, the proof of \Thm \ref{thm:RdmRepr} evaluates $\RdmCmpl_{\Measure, \NData} \qty(\LossHypothesis_{\Loss, \HypothesisSet'_{\Rate}})$, where $\HypothesisSet'_{\Rate}$ is given by replacing the condition $\Hypothesis \in \HypothesisSet$ in the definition of $\HypothesisSet_{\Rate}$ by $\sum_{\qty{\EntityI, \EntityII} \in \PairSet{\EntitySet}} \qty(\Hypothesis \qty(\qty{\EntityI, \EntityII}))^{2} \le \FuncLLBound^{2}$. 
Since $\HypothesisSet_{\Rate} \subset \HypothesisSet'_{\Rate}$, $\RdmCmpl_{\Measure, \NData} \qty(\LossHypothesis_{\Loss, \HypothesisSet'_{\Rate}}) \le \RdmCmpl_{\Measure, \NData} \qty(\LossHypothesis_{\Loss, \HypothesisSet'_{\Rate}})$ holds. 
Intuitively speaking, we allow any distance values that satisfy the condition about $\FuncLLBound$, regardless of whether they are actually achievable by the representations in $\ReprSpace$.
This leads to an easy local Rademacher complexity evaluation.
A potential issue here is that using $\HypothesisSet'_{\Rate}$ might be too conservative since this function set has ``forgot'' the information that the hypothesis function comes from the representation space $\ReprSpace$ and its distance function, other than it is restricted by $\FuncLLBound$. 
Hence it is possible that $\HypothesisSet_{\Rate}$ is no more than a very small part of $\HypothesisSet'_{\Rate}$.
If this is the case, we could improve our bound in the future.

\section*{Acknowledgements}
Atsushi Nitanda is partially supported by JSPS Kakenhi (22H03650).
Taiji Suzuki is partially supported by JST CREST.
Kenji Yamanishi is partially supported by JSPS Kakenhi (19H01114).

\bibliography{ref}
\bibliographystyle{icml2021}

\clearpage

\appendix
\twocolumn[\section*{\Large Supplementary Materials \\ for Generalization Error Bound for Hyperbolic Ordinal Embedding} \section*{}]


\section{Proof of \Thm \ref{thm:Main}}
We first confirm fundamental theorems to obtain excess risk bound from the Rademacher complexity.
\begin{corollary}[Corollary from Theorem 3 in \citep{DBLP:conf/nips/KakadeST08}]
\label{cor:GeneralGlobalBound}
Suppose that the conditions \labelcref{item:iid,item:LossSup} holds.
Fix $\ExceptProbability \in \Real_{> 0}$ and $\ERMMargin \in \Real_{\ge 0}$.
Then every $\ERMMargin$-CERM $\Learning: \qty(\FeatureSet \times \LabelSet)^{\NData} \to \Real$ satisfies
\begin{equation}
\begin{split}
& \SymClip{\Risk_{\Loss, \Measure}}{\ClipBound} \qty(\Learning \qty(\qty(\Point_{\IDatum})_{\IDatum=1}^{\NData})) - \inf \qty{\Risk_{\Loss, \Measure} \qty(\Hypothesis) \middle| \Hypothesis \in \HypothesisSet} 
\\
& \le 2 \RdmCmpl_{\Measure, \NData} \qty(\LossHypothesis_{\Loss, \HypothesisSet}) + 2 \LossBound \sqrt{\frac{\ln \frac{1}{\ExceptProbability}}{\NData}} + \ERMMargin,
\end{split}
\end{equation}
in probability at least $1 - \ExceptProbability$.
\end{corollary}
The convergence rate of the bound given by the above corollary \Cor \ref{cor:GeneralGlobalBound} is at the fastest $O \qty(\frac{1}{\sqrt{\NData}})$.
\Thm \ref{thm:Main} (ii) is derived using \Cor \ref{cor:GeneralGlobalBound}.

On the other hand, the other type of the excess risk bound, explained below, can give faster rate with some additional conditions. 
It uses the Rademacher complexity of a localized hypothesis function set, often called the \NewTerm{local Rademacher complexity} \citep{bartlett2005local, koltchinskii2006local}. The following is a simplified version of the version in \citep{steinwart2008support}.

\begin{corollary}[A simplified version of Theorem 7.~20 in \citep{steinwart2008support}]
\label{cor:GeneralLocalBound}
Let $\HypothesisSet \subset \mathcal{L}_{0} \qty(\FeatureSet)$ be equipped with a complete, separable metric dominating the pointwise convergence.
Assume that conditions \labelcref{item:Clippable,item:MarginBased,item:ReprLip,item:LossSup,item:BestHypothesis} and \eqref{eqn:OriginalVarBound} are satisfied and fix $\LipBound, \LossBound, \VarBound$ that satisfy the inequalities there.
Also, assume that there exists a Bayes decision function $\Hypothesis^{*} \in \mathcal{L}_{0} \qty(\FeatureSet)$, which satisfies $\SymClip{\Risk_{\Loss, \Measure}}{\ClipBound} \qty(\Hypothesis^{*}) = \SymClip{\Risk_{\Loss, \Measure}^{*}}{\ClipBound}$.
Define the approximation error $\ApproxError \DefEq \inf \qty{\SymClip{\Risk_{\Loss, \Measure}}{\ClipBound} \qty(\Hypothesis) - \SymClip{\Risk_{\Loss, \Measure}^{*}}{\ClipBound} \middle| \Hypothesis \in \HypothesisSet}$.
For $\Rate \ge \ApproxError$, define $\HypothesisSet_{\Rate} \DefEq \qty{\Hypothesis \in \HypothesisSet \middle| \SymClip{\Risk_{\Loss, \Measure}}{\ClipBound} \qty(\Hypothesis) - \SymClip{\Risk_{\Loss, \Measure}^{*}}{\ClipBound} \le \Rate}$.
Fix $\Hypothesis_{0} \in \HypothesisSet$ and $\LossBound_{0} > \sup \qty{\Loss \qty(\Feature, \Label, \Hypothesis_{0} \qty(\Feature)) \middle| \qty(\Feature, \Label) \in \FeatureSet \times \LabelSet} \lor \LossBound$.
Fix $\NData \in \Integer_{> 0}$, and assume that there exists a function $\RdmFunc_{\NData}: \Real_{\ge 0} \to \Real_{\ge 0}$ that satisfies $\RdmFunc_{\NData} \qty(4 \Rate) \le 2 \RdmFunc_{\NData} \qty(\Rate)$ and $\RdmFunc_{\NData} \qty(\Rate) \ge \RdmCmpl_{\Measure, \NData} \qty(\LossHypothesis_{\Loss, \HypothesisSet_{\Rate}})$.
Fix $\ExceptProbability \in \Real_{> 0}$, $\ERMMargin \in \Real_{\ge 0}$, and $\Rate \ge 30 \RdmFunc \qty(\Rate) \lor \qty(\frac{72 \VarBound \ln \frac{3}{\ExceptProbability}}{\NData})^{\frac{1}{2 - \VarExponent}} \lor \frac{5 \LossBound_{0} \ln \frac{3}{\ExceptProbability}}{\NData} \lor \ApproxError$.
Then every $\ERMMargin$-CERM $\Learning: \qty(\FeatureSet \times \LabelSet)^{\NData} \to \Real$ satisfies
\begin{equation}
\begin{split}
& \SymClip{\Risk_{\Loss, \Measure}}{\ClipBound} \qty(\Learning \qty(\qty(\Point_{\IDatum})_{\IDatum=1}^{\NData})) - \SymClip{\Risk_{\Loss, \Measure}^{*}}{\ClipBound} 
\\
& \le 6 \qty(\Risk_{\Loss, \Measure} \qty(\Hypothesis_{0}) - \Risk_{\Loss, \Measure}^{*}) + 3 \Rate + 3 \ERMMargin,
\end{split}
\end{equation}
in probability at least $1 - \ExceptProbability$.
\end{corollary}

\begin{proof}[Proof of \Thm \ref{thm:Main}]
Since $\FeatureSet$ is a finite sum, the expected risk is a finite weighted average of the loss.
Since the loss function $\Loss$, the function $\SomeFunc$, and the distance function $\Distance_{\ReprSpace}$ are all continuous from the assumption \labelcref{item:ReprLip,item:ContSome,item:ContDis}, we can regard the risk function $\Risk_{\Loss, \Measure}$ is a continuous real function on $\ReprSpace^{\abs{\EntitySet}}$.
Since $\ReprSpace^{\abs{\EntitySet}}$ is a compact topological space from the assumption \labelcref{item:compact}, the image of $\Risk_{\Loss, \Measure}$ is also compact. 
Hence, we have a 0-CERM. Since $\FeatureSet$ is a finite set, any map from $\FeatureSet^{\NData}$ is measurable. In particular, the 0-CERM is measurable. 
It implies the statement (i) of \Thm \ref{thm:Main}. 

The statement (ii) of \Thm \ref{thm:Main} is the direct consequence of \Cor \ref{cor:GeneralGlobalBound} if we admit \Thm \ref{thm:RdmRepr}, which we prove in the next section.

To prove the statement (iii) of \Thm \ref{thm:Main}, we need to show that $\HypothesisSet = \HypothesisSet_{\ReprSpace, \SomeFunc}$ is equipped with a complete, separable metric dominating the pointwise convergence.
Since $\FeatureSet$ is a finite set, we can regard $\HypothesisSet_{\ReprSpace, \SomeFunc} \subset \mathcal{L}_{0} \qty(\FeatureSet)$ as a subset of $\abs{\FeatureSet}$-dimensional vector space.
If we consider \IE a standard Euclidean metric in the $\abs{\FeatureSet}$-dimensional vector space, it is obvious that is dominates the pointwise convergence and $\HypothesisSet_{\ReprSpace, \SomeFunc}$ is separable by the metric.
Also, under the metric, the map $\ReprMap \mapsto \Hypothesis_{\ReprMap, \SomeFunc}$ is continuous from the continuity of $\SomeFunc$ and $\Distance_{\ReprMap}$. Here, we consider the topology of $\ReprSpace^{\EntitySet}$ by identifying it with $\ReprSpace^{\abs{\EntitySet}}$. 
Since  $\HypothesisSet_{\ReprSpace, \SomeFunc}$ is the image of the compact set $\ReprSpace^{\EntitySet}$ by the above continuous map, $\HypothesisSet_{\ReprSpace, \SomeFunc}$ is also compact.
This implies that $\HypothesisSet_{\ReprSpace, \SomeFunc}$ is complete (and totally bounded).

What remains to consider is the selection of $\RdmFunc_{\NData}$.
According to \Thm \ref{thm:RdmRepr}, we have that $\RdmCmpl_{\Measure, \NData} \qty(\LossHypothesis_{\Loss, \HypothesisSet_{\Rate}})
\le 
\frac{\UnRdmFunc \qty(\Rate)}{\sqrt{\NData}},
$,
where $\UnRdmFunc \qty(\Rate)
\DefEq
\min \qty{\UnRdmFunc_{\ICouple}\qty(\Rate) \middle| \ICouple = 0, 1, \dots, \PairSetCardinal{\EntitySet}}$.
However, since if we substitute $\RdmFunc_{\NData}$ with $\frac{\UnRdmFunc \qty(\Rate)}{\sqrt{\NData}}$, it does not satisfy $\RdmFunc_{\NData} \qty(4 \Rate) \le 2 \RdmFunc_{\NData} \qty(\Rate)$.

However, we can prove that for all $\NData \in \Integer_{> 0}$ and $\Rate_{0} \in \Real_{> 0}$, there exist $a, b, d \in \Real_{\ge 0}$ such that $\RdmFunc_{\NData} \qty(\Rate) \DefEq a \qty(\Rate^{\VarExponent} + d)^{\frac{1}{2 \VarExponent}} + b$ satisfies $\RdmFunc_{\NData} \qty(4 \Rate) \le 2 \RdmFunc_{\NData} \qty(\Rate)$, $\RdmFunc_{\NData} \qty(4 \Rate) \ge \frac{\UnRdmFunc \qty(\Rate)}{\sqrt{\NData}}$, and $\Rate \ge 30 \RdmFunc_{\NData} \qty(\Rate) \Leftrightarrow \Rate \ge 30 \frac{\UnRdmFunc \qty(\Rate)}{\sqrt{\NData}}$,
for all $\Rate \in \Real_{\ge 0}$.
Substituting such a $\RdmFunc_{\NData}$ and $\Hypothesis_{0} = \Hypothesis^{*} \in \HypothesisSet$ in \Cor \ref{cor:GeneralLocalBound}, we complete the proof of (iii) of \Thm \ref{thm:Main}. 
Note that $\Hypothesis^{*} \in \HypothesisSet$ is guaranteed by the condition \labelcref{item:BestHypothesis}.
\end{proof}

\section{Proof of \Thm \ref{thm:RdmRepr}}
\label{sec:RdmProof}
We review some basic properties of the Rademacher complexity.
\begin{lemma}
\label{lem:LinearRdm}
Let $c \in \Real$, $\HypothesisSet \subset \mathcal{L}_{0} \qty(\FeatureSet)$, and $\Hypothesis' \in \mathcal{L}_{0} \qty(\FeatureSet)$. Then,
\begin{equation}
\RdmCmpl_{\Measure, \NData} \qty(\qty{c \Hypothesis + \Hypothesis' \middle| \Hypothesis \in \HypothesisSet}) = \abs{c} \RdmCmpl_{\Measure, \NData} \qty(\HypothesisSet). 
\end{equation}
\end{lemma}
For the proof of \Lem \ref{lem:LinearRdm}, see \citep[\EG Lemma 26.6,][]{shalev2014understanding}.

\begin{lemma}
\label{lem:RdmContract}
Let $\ReprFunc: \Real \to \Real$ be a Lipschitz continuous function and $\HypothesisSet \subset \mathcal{L}_{0} \qty(\FeatureSet)$. Then, $\RdmCmpl_{\Measure, \NData} \qty(\qty{\ReprFunc \circ \Hypothesis \middle| \Hypothesis \in \HypothesisSet}) = \Lipschitz \qty(\ReprFunc) \RdmCmpl_{\Measure, \NData} \qty(\HypothesisSet)$, where $\Lipschitz \qty(\ReprFunc) \in \Real_{\ge 0}$ is the Lipschitz constant of $\ReprFunc$.
\end{lemma}
For the proof of \Lem \ref{lem:RdmContract}, see \citep[\EG Lemma 26.9,][]{shalev2014understanding}.
The following is easy using \Lem \ref{lem:RdmContract}.
\begin{lemma}
\label{lem:LossRdm}
Suppose that the conditions \labelcref{item:MarginBased,item:ReprLip} hold and $\LipConst$ is a constant that satisfies the inequality in \labelcref{item:ReprLip}.
\begin{equation}
\begin{split}
\RdmCmpl_{\Measure, \NData} \qty(\LossHypothesis_{\Loss, \HypothesisSet}) \le \LipConst \RdmCmpl_{\Measure_{\FeatureSet}, \NData} \qty(\HypothesisSet).
\end{split}
\end{equation}
\end{lemma}

\begin{proof}[Proof of \Thm \ref{thm:RdmRepr}]
We regard every element in $\HypothesisSet_{\ReprSpace, \SomeFunc}$ as a $\PairSetCardinal{\EntitySet}$-dimensional vector as follows.
First, we fix an index map $\Index: \qty{1, 2, \dots, \PairSetCardinal{\EntitySet}} \to \PairSet{\EntitySet}$.
We can use any map as $\Index$ as long as it is bijective.
In the following, for a vector $\VecI$, we denote the $\ICouple$-th element by $\qty[\VecI]_{\ICouple}$.
We define $\HypothesisVec_{\ReprMap, \SomeFunc} \in \Real^{\PairSetCardinal{\EntitySet}}$ by $\qty[\HypothesisVec_{\ReprMap, \SomeFunc}]_{\ICouple} = \Hypothesis_{\ReprMap, \SomeFunc} \qty(\Index \qty(\ICouple))$.
Also, define $\OnehotVec_{\qty{\EntityI, \EntityII}} \in \Real^{\PairSetCardinal{\EntitySet}}$ by 
\begin{equation}
\qty[\OnehotVec_{\qty{\EntityI, \EntityII}}]_{\ICouple} = 
\begin{cases}
    1 & \qif{\Index \qty(\ICouple) = \qty{\EntityI, \EntityII}},
    \\
    0 & \qif{\Index \qty(\ICouple) \ne \qty{\EntityI, \EntityII}}.
\end{cases}
\end{equation}
Since $\HypothesisVec_{\ReprMap, \SomeFunc}^{\Transpose} \OnehotVec_{\qty{\EntityI, \EntityII}} = \Hypothesis_{\ReprMap, \SomeFunc} \qty(\qty{\EntityI, \EntityII})$, we can identify $\HypothesisVec_{\ReprMap, \SomeFunc}$ and $\FeatureVec_{\qty{\EntityI, \EntityII}}$ with $\Hypothesis_{\ReprMap, \SomeFunc}$ and $\qty{\EntityI, \EntityII}$, respectively.
For $\Hypothesis: \FeatureSet \to \Real$, we define $\SymClip{\Hypothesis}{\ClipBound}: \FeatureSet \to \Real$ by $\SymClip{\Hypothesis}{\ClipBound} \qty(\Feature) = \SymClip{\Hypothesis \qty(\Feature)}{\ClipBound}$.

Recall
\begin{equation}
\begin{split}
\HypothesisSet_{\Rate}
=
\qty{\SymClip{\Hypothesis_{\ReprMap, \SomeFunc}}{\ClipBound} \middle| 
\begin{aligned}
\SymClip{\Risk_{\Loss, \Measure}}{\ClipBound} \qty(\Hypothesis) - \SymClip{\Risk_{\Loss, \Measure}^{*}}{\ClipBound} & \le \Rate.
\end{aligned}
},
\end{split}
\end{equation}
and
\begin{equation}
\begin{split}
\HypothesisSet'_{\Rate}
=
\qty{\SymClip{\Hypothesis}{\ClipBound} \middle| 
\begin{aligned}
\Hypothesis: \FeatureSet &\to \Real,
\\
\sum_{\Feature \in \PairSet{\EntitySet}} \qty(\Hypothesis \qty(\Feature))^{2} & \le \FuncLLBound^{2}, \\ \SymClip{\Risk_{\Loss, \Measure}}{\ClipBound} \qty(\Hypothesis) - \SymClip{\Risk_{\Loss, \Measure}^{*}}{\ClipBound} & \le \Rate.
\end{aligned}
}.
\end{split}
\end{equation}
Here, we have $\HypothesisSet \subset \HypothesisSet'$.

Define
\begin{equation}
\begin{split}
\HypothesisSet_{\Rate}^{(2)}
\DefEq
\qty{\SymClip{\Hypothesis}{\ClipBound} \middle| 
\begin{aligned}
\Hypothesis: \FeatureSet &\to \Real,
\\
\sum_{\Feature \in \PairSet{\EntitySet}} \qty(\Hypothesis \qty(\Feature))^{2} & \le \FuncLLBound^{2}, \\ \Expect_{\Feature \sim \Measure_{\FeatureSet}} \qty[\SymClip{\Hypothesis \qty (\Feature)}{\ClipBound} - \SymClip{\Hypothesis^{*} \qty (\Feature)}{\ClipBound}]^{2} & \le \ConvexBound \Rate^{\VarExponent}.
\end{aligned}
},
\end{split}
\end{equation}
then $\HypothesisSet'_{\Rate} \subset \HypothesisSet_{\Rate}^{(2)}$ follows the condition \labelcref{item:LossVar}.

Using the vector notation, we have that
\begin{equation}
\begin{split}
\HypothesisSet_{\Rate}^{(2)}
& =
\qty{\SymClip{\HypothesisVec}{\ClipBound}^\Transpose \OnehotVec_{\qty(\cdot)} \middle| 
\begin{aligned}
\HypothesisVec^\Transpose \HypothesisVec & \le \FuncLLBound^{2}, \\ \Expect_{\Feature \sim \Measure_{\FeatureSet}} \qty[\SymClip{\HypothesisVec^\Transpose \OnehotVec_{\Feature}}{\ClipBound} - \SymClip{{\HypothesisVec^{*}}^\Transpose \OnehotVec_{\Feature}}{\ClipBound}]^{2} & \le \ConvexBound \Rate^{\VarExponent}
\end{aligned}
}
\\
& =
\qty{\SymClip{\HypothesisVec}{\ClipBound}^\Transpose \OnehotVec_{\qty(\cdot)} \middle| 
\begin{aligned}
\HypothesisVec^\Transpose \HypothesisVec & \le \FuncLLBound^{2}, \\ \Expect_{\Feature \sim \Measure_{\FeatureSet}} \qty[\SymClip{\HypothesisVec}{\ClipBound}^\Transpose \OnehotVec_{\Feature} - \SymClip{{\HypothesisVec^{*}}}{\ClipBound}^\Transpose \OnehotVec_{\Feature}]^{2} & \le \ConvexBound \Rate^{\VarExponent}
\end{aligned}
},
\end{split}
\end{equation}
where we define
$\HypothesisVec^{*} \in \Real^{\PairSetCardinal{\EntitySet}}$ by $\qty[\HypothesisVec^{*}]_{\ICouple} = \Hypothesis^{*} \qty(\Index \qty(\ICouple))$ and for $\HypothesisVec \in \Real^{\PairSetCardinal{\EntitySet}}$ we define $\SymClip{\HypothesisVec}{\ClipBound} \in \Real^{\PairSetCardinal{\EntitySet}}$ by $\qty[\SymClip{\HypothesisVec}{\ClipBound}]_{\ICouple} = \SymClip{\qty[\HypothesisVec]_{\ICouple}}{\ClipBound}$.

Since $\HypothesisVec^\Transpose \HypothesisVec \le \FuncLLBound^{2} \Rightarrow \SymClip{\HypothesisVec}{\ClipBound}^\Transpose \SymClip{\HypothesisVec}{\ClipBound} \le \FuncLLBound^{2}$,
we have that 
$\HypothesisSet_{\Rate}^{(2)} \subset \HypothesisSet_{\Rate}^{(3)}$, where $\HypothesisSet_{\Rate}^{(3)}$ is defined by
\begin{equation}
\begin{split}
\HypothesisSet_{\Rate}^{(3)}
& \DefEq
\qty{\HypothesisVec^\Transpose \OnehotVec_{\qty(\cdot)} \middle| 
\begin{aligned}
\HypothesisVec^\Transpose \HypothesisVec & \le \FuncLLBound^{2}, \\ \Expect_{\Feature \sim \Measure_{\FeatureSet}} \qty[\HypothesisVec^\Transpose \OnehotVec_{\Feature} - {\HypothesisVec^{*}}^\Transpose \OnehotVec_{\Feature}]^{2} & \le \ConvexBound \Rate^{\VarExponent}
\end{aligned}
}.
\end{split}
\end{equation}
By \Lem \ref{lem:LinearRdm}, we have $\RdmCmpl_{\Measure_{\FeatureSet}, \NData} \qty(\HypothesisSet_{\Rate}^{(3)}) = \RdmCmpl_{\Measure_{\FeatureSet}, \NData} \qty(\HypothesisSet_{\Rate}^{(4)})$, where $\HypothesisSet_{\Rate}^{(4)}$ is given by
\begin{equation}
\HypothesisSet_{\Rate}^{(4)}
\DefEq
\qty{\qty(\HypothesisVec - \HypothesisVec^{*})^\Transpose \OnehotVec_{\qty(\cdot)} \middle| 
\begin{aligned}
\HypothesisVec^\Transpose \HypothesisVec & \le \FuncLLBound^{2}, \\ \Expect_{\Feature \sim \Measure_{\FeatureSet}} \qty[\HypothesisVec^\Transpose \OnehotVec_{\Feature} - {\HypothesisVec^{*}}^\Transpose \OnehotVec_{\Feature}]^{2} & \le \ConvexBound \Rate^{\VarExponent}
\end{aligned}
}
\\
\subset
\end{equation}

We can evaluate the above set as follows.
\begin{equation}
\begin{split}
\HypothesisSet_{\Rate}^{(4)}
& \subset 
\qty{\qty(\HypothesisVec - \HypothesisVec')^\Transpose \OnehotVec_{\qty(\cdot)} \middle| 
\begin{aligned}
\HypothesisVec^\Transpose \HypothesisVec \le \FuncLLBound^{2}, \HypothesisVec'^\Transpose \HypothesisVec' & \le \FuncLLBound^{2}, \\ \Expect_{\Feature \sim \Measure_{\FeatureSet}} \qty[\qty(\HypothesisVec - \HypothesisVec')^\Transpose \OnehotVec_{\Feature}]^{2} & \le \ConvexBound \Rate^{\VarExponent}
\end{aligned}
}
\\
& =
\qty{2 \HypothesisVec^\Transpose \OnehotVec_{\qty(\cdot)} \middle| 
\begin{aligned}
\HypothesisVec^\Transpose \HypothesisVec & \le \FuncLLBound^{2}, \\ \Expect_{\Feature \sim \Measure_{\FeatureSet}} \qty[2 \HypothesisVec^\Transpose \OnehotVec_{\Feature}]^{2} & \le \ConvexBound \Rate^{\VarExponent}
\end{aligned}
}
\\
& =
\qty{2 \FuncLLBound \HypothesisVec^\Transpose \OnehotVec_{\qty(\cdot)} \middle| 
\begin{aligned}
\HypothesisVec^\Transpose \HypothesisVec & \le 1, \\ \Expect_{\Feature \sim \Measure_{\FeatureSet}} \qty[\HypothesisVec^\Transpose \OnehotVec_{\Feature}]^{2} & \le \frac{\ConvexBound \Rate^{\VarExponent}}{4 \FuncLLBound^{2}}
\end{aligned}
}.
\end{split}
\end{equation}
Hence, by defining $\widehat{\HypothesisSet_{\Rate}}$ as 
\begin{equation}
\widehat{\HypothesisSet_{\Rate}}
\DefEq \qty{\HypothesisVec^\Transpose \OnehotVec_{\qty(\cdot)} \middle| 
\begin{aligned}
\HypothesisVec^\Transpose \HypothesisVec & \le 1, \\ \Expect_{\Feature \sim \Measure_{\FeatureSet}} \qty[\HypothesisVec^\Transpose \OnehotVec_{\Feature}]^{2} & \le \frac{\ConvexBound \Rate^{\VarExponent}}{4 \FuncLLBound^{2}}
\end{aligned}
},
\end{equation}
we have that 
$\RdmCmpl_{\Measure_{\FeatureSet}, \NData} \qty(\HypothesisSet_{\Rate}^{(4)}) \le 2 \FuncLLBound \RdmCmpl_{\Measure_{\FeatureSet}, \NData} \qty(\widehat{\HypothesisSet_{\Rate}})$ from \Lem \ref{lem:RdmContract}.

We apply Theorem 41 in \citep{mendelson2002geometric}.
The following is the version in \citep{bartlett2005local} given as the first half of Theorem 6.5.
\begin{theorem}[The first half of Theorem 6.5 in \citep{bartlett2005local}, given in Theorem 41 in \citep{mendelson2002geometric}for the first time.]
\label{thm:EigenRdm}
Let $\FeatureSet$ be a measurable set and $\Measure$ is a distribution on it.
Let $\KernelFunc: \FeatureSet \times \FeatureSet \to \Real$ be a positive semidefinite kernel function that satisfies $\Expect_{\Feature \sim \Measure} \KernelFunc \qty(\Feature, \Feature) < + \infty$.
Define the integral operator $\IntegralOp: \mathcal{L}_{2} \qty(\Measure) \to \mathcal{L}_{2} \qty(\Measure)$ by $\qty(\IntegralOp \qty(\Hypothesis)) \qty(\Feature) \DefEq \Expect_{\Feature' \sim \Measure} \KernelFunc \qty(\Feature, \Feature') \Hypothesis \qty(\Feature')$ and let $\qty(\EigenVal_{\ICouple})_{\ICouple=1}^{\infty}$ be the sequence of the eigenvalues of $\IntegralOp$.
Let $\RKHS_{\KernelFunc}$ be the reproducing kernel Hilbert space generated by $\KernelFunc$ and denote its norm function by $\norm{\cdot}_{\RKHS_{\KernelFunc}}$.
Then, 
\begin{equation}
\begin{split}
& \RdmCmpl_{\Measure, \NData} \qty(\qty{\Hypothesis \in \RKHS_{\KernelFunc} \middle| \norm{\Hypothesis}_{\RKHS_{\KernelFunc}} \le 1, \Expect_{\Feature \sim \Measure} \qty(\Hypothesis \qty(\Feature))^{2} \le \LocalRadius}) 
\\
& \le \sqrt{\frac{2}{\NData} \sum_{\ICouple=1}^{\infty} \min \qty{\LocalRadius, \EigenVal_{\ICouple}}}.
\end{split}
\end{equation}
\end{theorem}
Here, we consider the linear kernel function $\KernelFunc \qty(\FeatureVec, \FeatureVec) = \FeatureVec^\Transpose \FeatureVec$. 
Then we can easily confirm $\norm{\HypothesisVec^\Transpose \OnehotVec_{\qty(\cdot)}}_{\RKHS_{\KernelFunc}} = \sqrt{\HypothesisVec^\Transpose \HypothesisVec}$, and $\IntegralOp$ is given by the matrix $\sum_{\ICouple = 1}^{\PairSetCardinal{\EntitySet}} \Measure_{\FeatureSet} \qty(\qty{\Index \qty(\ICouple)}) \OnehotVec_{\qty(\Index \qty(\ICouple))} \OnehotVec_{\qty(\Index \qty(\ICouple))}^\Transpose$.
Hence, we have that
\begin{equation}
\EigenVal_{\ICouple} =
\begin{cases}
\Measure_{\FeatureSet} \qty(\qty{\Index \qty(\ICouple)}) & \qif{\ICouple = 1, 2, \dots, \PairSetCardinal{\EntitySet}},
\\
0 & \qif{\ICouple > \PairSetCardinal{\EntitySet}}.
\end{cases}
\end{equation}
Applying the above and using \Lem \ref{lem:LossRdm}, we complete the proof.
\end{proof}

\section{Proof of \Lem \ref{lem:EuclideanBound}}
\begin{proof}[Proof of \Lem \ref{lem:EuclideanBound}]
We first prove it for $\DistanceExponent = 1$.
Let $\ReprVec_{\Entity} \in \Real^{\NAxes}$ be the representation of $\Entity \in \EntitySet$. 
Fix a bijective map $\Index: \qty{1, 2, \dots, \PairSetCardinal{\EntitySet}} \to \PairSet{\EntitySet}$, which we call an indexing map.
We define the representation matrix $\ReprMat \in \Real^{\NAxes, \abs{\EntitySet}}$ by
$\ReprMat \DefEq \mqty[\ReprVec_{\Index \qty(1)} & \ReprVec_{\Index \qty(2)} & \cdots & \ReprVec_{\Index \qty(\abs{\EntitySet})}]$.
Then,
\begin{equation}
\begin{split}
& \sum_{\qty{\EntityI, \EntityII} \in \PairSetCardinal{\EntitySet}} \qty(\Distance_{\ReprSpace} \qty(\ReprMap \qty(\EntityI), \ReprMap \qty(\EntityII)))^{2}
\\
& = \frac{1}{2} \Tr{\ReprMat^\Transpose \ReprMat \qty(\abs{\EntitySet} \Identity_{\abs{\EntitySet}} - \OneVec_{\abs{\EntitySet}} \OneVec_{\abs{\EntitySet}}^{\Transpose})}
\\
& \le \frac{1}{2} \Tr{\ReprMat^\Transpose \ReprMat \qty(\abs{\EntitySet} \Identity_{\abs{\EntitySet}})}
\\
& = \frac{\abs{\EntitySet}}{2} \Tr{\ReprMat^\Transpose \ReprMat} 
\\
& = \frac{\abs{\EntitySet}}{2} \sum_{\Entity \in \EntitySet} \ReprVec_{\Entity}^{\Transpose} \ReprVec_{\Entity},
\\
& \le \frac{\abs{\EntitySet}}{2} \Radius^{2}.
\end{split}
\end{equation}

If $\DistanceExponent > 1$, it follows that
\begin{equation}
\label{eqn:HighExponent}
\begin{split}
& \sum_{\qty{\EntityI, \EntityII} \in \PairSetCardinal{\EntitySet}} \qty(\Distance_{\ReprSpace} \qty(\ReprMap \qty(\EntityI), \ReprMap \qty(\EntityII)))^{2 \DistanceExponent}
\\
& \le
\sum_{\qty{\EntityI, \EntityII} \in \PairSetCardinal{\EntitySet}} \qty(\Distance_{\ReprSpace} \qty(\ReprMap \qty(\EntityI), \ReprMap \qty(\EntityII)))^{2} \qty(2 \Radius)^{2 (\DistanceExponent - 1)},
\end{split}
\end{equation}
which completes the proof.
\end{proof}

\section{Proof of \Lem \ref{lem:HyperbolicBound}}
\begin{proof}[Proof of \Lem \ref{lem:HyperbolicBound}]
First, we prove it for $\DistanceExponent = 1$.
Since $\NAxes \ge 2$, the space contains a two dimensional hyperbolic disk as a subspace.
In the hyperbolic disk, consider a regular polygon centered at the origin with $\abs{\EntitySet}$ vertices and radius $\Radius$.
Using the hyperbolic law of sines, we have that the length of one side in the polygon is given by $2 \Asinh \qty(\sin \frac{\pi}{\abs{\EntitySet}} \sinh \Radius)$.
Since $\frac{2 \Asinh \qty(\sin \frac{\pi}{\abs{\EntitySet}} \sinh \Radius)}{\Radius} \to 2$ as $\Radius \to \infty$, we obtain the consequence of the lemma for $\DistanceExponent = 1$.
For $\DistanceExponent > 1$, we obtain the consequence by \eqref{eqn:HighExponent}, which completes the proof.
\end{proof}

\section{Hinge loss and \Cor \ref{cor:hinge}}
In this section, we just confirm that \Cor \ref{cor:hinge} immediately follows \Thm \ref{thm:Main} and the following existing theorem.
\begin{theorem}[Theorem 8.24 in \citep{steinwart2008support}]
Let $\Measure$ be a distribution on $\FeatureSet \times \qty{\pm 1}$ and the loss function be the hinge loss $\Loss_\mathrm{hinge} \qty(\Feature, \Label, \Prediction) \DefEq \ReprFunc_\mathrm{hinge} \qty(\Label \Prediction)$, where $\ReprFunc_\mathrm{hinge} \qty(\Prediction') \DefEq \max \qty{1 - \Prediction', 0}$ with $\ClipBound = 1$. Define the risk function $\Risk_{\Loss, \Measure}$ as in \Sub \ref{sub:risk}. 
Assume that the distribution $\Measure$ has noise exponent $\NoiseExponent \in \Real_{\ge 0}$ with constant $\NoiseConst \in \Real_{>0}$. Then, for all $\Hypothesis \in \mathcal{L}_{0} \qty(\FeatureSet)$, then the condition \ref{item:LossVar} in \Asp \ref{asp:Main} holds with $\VarExponent = \frac{\NoiseExponent}{\NoiseExponent + 1}$ and $\ConvexBound = 6 \NoiseConst^{\frac{\NoiseExponent}{\NoiseExponent + 1}}$.
\end{theorem}

\section{General condition for hyperbolic to outperform Euclidean}
\label{sec:GeneralCondition}
In \Exm \ref{exm:Numerical}, we gave the condition for hyperbolic graph embedding to outperform Euclidean graph embedding on a specific setting.
We give the condition for a general setting in the following, which we can obtain by simple calculation from \Thm \ref{thm:Main}.
\begin{proposition}
Suppose that conditions \labelcref{item:iid,item:compact,item:ContDis,item:ContSome} in \Asp \ref{asp:Main} are satisfied, the loss function be the hinge loss, and $\SomeFunc \qty(\Prediction) = \Prediction^{\DistanceExponent}$.
Let the true dissimilarity $\Dsim^{*}: \VertexSet \times \VertexSet \to \Real_{\ge 0}$ be given by the graph distance of a tree.
Then, for $\Radius$ given by \Lem \ref{lem:MarginEmbedding}, the expected risk of a CERM using $\Ball [\Radius; \Hyperbolic^{2}]$ is better than any CERM using $\Euclid^{2}$ in probability at least $1 - \ExceptProbability$ if $\NData \ge \qty(\MajorNData_{0} \land \MajorNData_{\PairSetCardinal{\EntitySet}}) \lor \MinorNData$, where
\begin{equation}
\begin{split}
\MajorNData_{0}
& \DefEq 
97200 \qty[\DistanceExponent \qty(2 \Radius)^{\DistanceExponent - 1}]^{2} \PairSetCardinal{\EntitySet}^{2} \frac{1}{\NoiseMargin^{2} \NViolate_{\MinSymb} \qty(\Radius; \Euclid^{2})},
\\
\MajorNData_{\PairSetCardinal{\EntitySet}}
& \DefEq
32 \Radius^{2} \qty[\DistanceExponent \qty(2 \Radius)^{\DistanceExponent - 1}]^{2} \PairSetCardinal{\EntitySet}^{2} \frac{1}{\NoiseMargin^2 \qty[\NViolate_{\MinSymb} \qty(\Radius; \Euclid^{2})]^{2}},
\\
\MinorNData
& \DefEq
3888 \frac{1}{\NoiseMargin^{2} \NViolate_{\MinSymb} \qty(\Radius; \Euclid^{2})} \ln \frac{3}{\ExceptProbability}.
\end{split}
\end{equation}
\end{proposition}

\section{The definition of $\EdgeMatVarNorm$ and dependency of the bounds by \Thm \ref{thm:Main} and \Cor \ref{cor:OldBound} on $\abs{\EntitySet}$.}
\label{sec:OldDetail}
\newcommand{\EdgeMat}{\Mat{E}}
The value $\EdgeMatVarNorm$, which the bound in \Cor \ref{cor:OldBound} depends on, is defined in \citep{suzuki2021graph} as $\EdgeMatVarNorm \DefEq \norm{\Expect_{\qty{\EntityI, \EntityII} \sim \Measure_{\FeatureSet}} \EdgeMat_{\qty{\EntityI, \EntityII}}^{2}}_{\OpSymb, 2}$, where the symmetric matrix $\EdgeMat_{\qty{\EntityI, \EntityII}} $ for $\qty{\EntityI, \EntityII} \in \PairSet{\EntitySet}$ is given by 
\begin{equation}
\qty[\EdgeMat_{\qty{\EntityI, \EntityII}}]_{i, j}
=
\begin{cases}
c_\mathrm{diag} & \qif{\qty{\Index \qty(i), \Index \qty(j)} \subsetneq \qty{\EntityI, \EntityII}},
\\
c_\mathrm{off} & \qif{\qty{\Index \qty(i), \Index \qty(j)} = \qty{\EntityI, \EntityII}},
\\
0 & \qif{\qty{\Index \qty(i), \Index \qty(j)} \not\subset \qty{\EntityI, \EntityII}}.
\end{cases}
\end{equation}
Here $\qty(c_\mathrm{diag}, c_\mathrm{off}) = \qty(1, -1)$ for the Euclidean case, and $\qty(c_\mathrm{diag}, c_\mathrm{off}) = \qty(0, -\frac{1}{2})$.
Here, $\norm{\cdot}_{\OpSymb, 2}$ is the operator norm with respect to 2-norm.
For a real symmetric matrix $\MatI$, $\norm{\MatI}_{\OpSymb, 2}$ equals to the maximum eigenvalue of $\MatI$ and also equals to the maximum singular value of $\MatI$.
We have that
\begin{equation}
\qty[\EdgeMat_{\qty{\EntityI, \EntityII}}^{2}]_{i, j}
=
\begin{cases}
c'_\mathrm{diag} & \qif{\qty{\Index \qty(i), \Index \qty(j)} \subsetneq \qty{\EntityI, \EntityII}},
\\
c'_\mathrm{off} & \qif{\qty{\Index \qty(i), \Index \qty(j)} = \qty{\EntityI, \EntityII}},
\\
0 & \qif{\qty{\Index \qty(i), \Index \qty(j)} \not\subset \qty{\EntityI, \EntityII}},
\end{cases}
\end{equation}
where $\qty(c_\mathrm{diag}, c_\mathrm{off}) = \qty(2, -2)$ for the Euclidean case, and $\qty(c_\mathrm{diag}, c_\mathrm{off}) = \qty(\frac{1}{4}, 0)$.
For the upper bound of $\EdgeMatVarNorm$, as pointed out by \citep{suzuki2021graph}, we have that $\EdgeMatVarNorm \DefEq \norm{\Expect_{\qty{\EntityI, \EntityII} \sim \Measure_{\FeatureSet}} \EdgeMat_{\qty{\EntityI, \EntityII}}^{2}}_{\OpSymb, 2} \le \Expect_{\qty{\EntityI, \EntityII} \sim \Measure_{\FeatureSet}} \norm{ \EdgeMat_{\qty{\EntityI, \EntityII}}^{2}}_{\OpSymb, 2}$ from Jensen's inequality. The right side is 4 for the Euclidean case and $\frac{1}{4}$ for the hyperbolic case.
Indeed, these upper bounds are achievable if only one couple of entities is generated.
For the lower bound, we can see that the trace of $\EdgeMat_{\qty{\EntityI, \EntityII}}^{2}$ is always $4$ for the Euclidean case and $\frac{1}{2}$ for hyperbolic case, as we can see by summing the diagonal elements up.
Hence, it also holds for its expectation $\Expect_{\qty{\EntityI, \EntityII} \sim \Measure_{\FeatureSet}} \EdgeMat_{\qty{\EntityI, \EntityII}}^{2}$.
We remark that the trace equals to the sum of eigenvalues. 
Since we have $\PairSetCardinal{\EntitySet}$ eigenvalues, the mean of eigenvalues is $\frac{4}{\PairSetCardinal{\EntitySet}}$ for the Euclidean case and $\frac{1}{2 \PairSetCardinal{\EntitySet}}$ for the hyperbolic case.
The value $\EdgeMatVarNorm$ is the maximum in the eigenvalues, which is not smaller than the mean. 
Hence, $\EdgeMatVarNorm$ is  lower-bounded by $\frac{4}{\PairSetCardinal{\EntitySet}}$ for the Euclidean case and $\frac{1}{2 \PairSetCardinal{\EntitySet}}$ for the hyperbolic case.
For both cases, the lower-bound is achieved by the uniform distribution.

Let us consider the bound by \Cor \ref{cor:OldBound} again.
If we focus on $\abs{\EntitySet}$ and $\NData$, the bound is
$O \qty(\frac{\abs{\EntitySet} \sqrt{\EdgeMatVarNorm \ln \abs{\EntitySet}}}{\sqrt{\NData}} + \frac{\EdgeMatSV \abs{\EntitySet} \ln \abs{\EntitySet}}{\NData})$.
For the upper bound case, \Cor \ref{cor:OldBound} gives $O \qty(\frac{\abs{\EntitySet} \sqrt{\ln \abs{\EntitySet}}}{\sqrt{\NData}} + \frac{\EdgeMatSV \abs{\EntitySet} \ln \abs{\EntitySet}}{3 \NData})$. Since \Thm \ref{thm:RdmRepr} gives the bound that is $O \qty(\frac{\sqrt{\abs{\EntitySet}}}{\sqrt{\NData}})$ for the Euclidean case and $O \qty(\frac{\abs{\EntitySet}}{\sqrt{\NData}})$ for the hyperbolic case, \Thm \ref{thm:RdmRepr} is better than \Cor \ref{cor:OldBound}.
For the lower bound case, \Cor \ref{cor:OldBound} gives $O \qty(\frac{\sqrt{\ln \abs{\EntitySet}}}{\sqrt{\NData}} + \frac{\EdgeMatSV \abs{\EntitySet} \ln \abs{\EntitySet}}{3 \NData})$.
Here, the dependency on $\abs{\EntitySet}$ is significantly different between the first and second term.
It implies that if $\NData$ is sufficiently large, then \Cor \ref{cor:OldBound} is better in the dependency on $\abs{\EntitySet}$ than \Thm \ref{thm:RdmRepr}, while the converse holds if $\NData$ is not large. 

\end{document}